\documentclass[lettersize,journal]{IEEEtran}
\usepackage{amsmath,amsfonts}
\usepackage{algorithmic}
\usepackage{algorithm}
\usepackage{array}
\usepackage[caption=false,font=normalsize,labelfont=sf,textfont=sf]{subfig}
\usepackage{textcomp}
\usepackage{stfloats}
\usepackage{url}
\usepackage{verbatim}
\usepackage{graphicx}
\usepackage{cite}

\usepackage{color}
\usepackage{amsmath,lipsum}
\usepackage{tabularx}
\usepackage{amsmath,amsfonts}
\usepackage{multirow}
\usepackage{widetext}
\usepackage{makecell}

\usepackage{amsmath}
\usepackage{amssymb}
\usepackage{amsthm}

\newtheorem{theorem}{Theorem} 
\newtheorem{lemma}{Lemma}
\newtheorem{corollary}{Corollary}
\newtheorem{definition}{Definition}

\begin{document}

\title{Dualformer: Efficient Feature Extractor for Complex-valued Blind Communication Signal Analysis}

\author{Yurui Zhao, Xiang Wang, Jingreng Lei, Wanlong Zhang, Yik-Chung Wu,~\IEEEmembership{Senior Member, IEEE}, Zhitao Huang
\thanks{Yurui Zhao, Xiang Wang, Wanlong Zhang, and Zhitao Huang are with College of Electronic Science and Technology, National University of Defense Technology, Changsha 410073, China.}
\thanks{Jingreng Lei and Yik-Chung Wu are with the Department of Electrical and Computer Engineering, The University of Hong Kong (E-mail: leijr@eee.hku.hk;ycwu@eee.hku.hk).}
\thanks{Corresponding author: Xiang Wang (E-mail: xwang@nudt.edu.cn)}
\thanks{Research is supported by the National Natural Science Foundation of China, Grant No.62271494.}
\thanks{Manuscript received December 31, 2023.}}

\markboth{Journal of \LaTeX\ Class Files,~Vol.~14, No.~8, August~2021}%
{Shell \MakeLowercase{\textit{et al.}}: A Sample Article Using IEEEtran.cls for IEEE Journals}

\maketitle

\begin{abstract}
Designing effective feature extractors is critical for blind signal analysis tasks such as automatic modulation recognition (AMR), signal scheme recognition (SSR), and \color{black} signal structure parsing (SSP). 
\color{black}
In this work, we propose dual-channel neural network (DualNN) that efficiently exploits complex-valued signals through parameter sharing across IQ channels. Unlike traditional real-valued or complex-valued models, DualNN is a groundbreaking framework which shares the network parameters for processing the real and imaginary parts of the complex-valued signals, and is theoretically shown to reduce generalization error while preserving expressive capacity. Specifically, we propose a novel Transformer-based architecture to implement DualNN, called Dualformer. 
The Dualformer segments input signals into patch-level tokens and captures multi-granularity features, enabling robust performance across diverse signal analysis tasks.
Furthermore, we conduct extensive experiments comparing Dualformer with three Transformer-based baselines and four conventional DL-based approaches. 
Results demonstrate consistent performance improvements on AMR, SSR, and SSP tasks. 
Besides, the modular design of DualNN allows it to generalize well to blind signal processing tasks such as blind source separation and low-SNR spectrum sensing. 
This work paves the way for a broader application of DualNN architectures in unsupervised and weakly supervised complex-valued signal analysis scenarios.
\end{abstract}

\begin{IEEEkeywords}
Automatic modulation recognition, signal scheme recognition, 
\color{black}
signal structure parsing,
\color{black}
complex-valued signal, transformer, generalization error.
\end{IEEEkeywords}

\section{Introduction}
\IEEEPARstart{B}{lind} analysis of communication signals has occupied an increasingly important position with the rapid advancement of integrated communication and sensing technologies \cite{liu2022integrated}, spectrum monitoring \cite{de2008blind}, cognitive radio \cite{awin2018blind}, and the upcoming 6G communication systems \cite{chowdhury20206g, 11431942}.
Blind signal analysis is defined as the process of performing signal modulation parameter estimation and related analyses using only the received communication signal data, without detailed knowledge of the signal characteristics. It enables critical post-processing functionalities such as information recovery, target recognition, and jamming guidance.

\begin{figure*}[!t]
\centering
\includegraphics[width=0.9\linewidth]{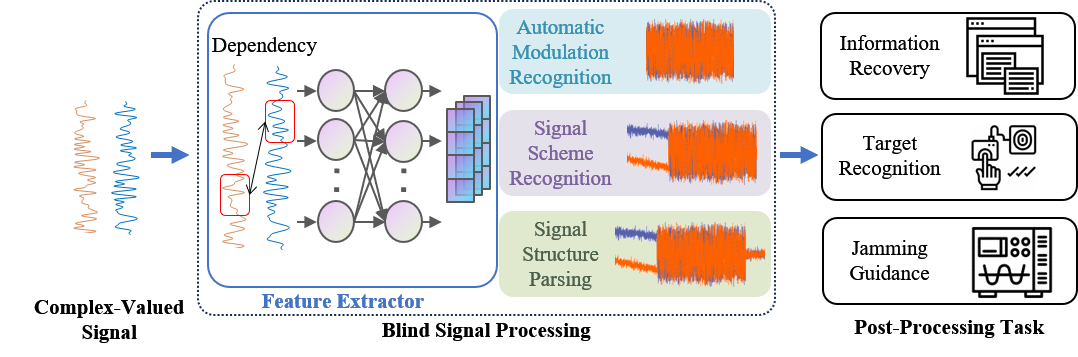}
\caption{Schematic plot of the feature extractor in blind communication signal analysis. }
\label{fig_task}
\end{figure*}

Currently, blind analysis of communication signals primarily focuses on three different tasks: automatic modulation recognition (AMR), signal scheme recognition (SSR), and signal structure parsing (SSP), as illustrated in Fig.~\ref{fig_task}.
AMR and SSR identify the modulation type and signal scheme (e.g., Wi-Fi, Bluetooth), respectively, both based on the received signal data. SSP can segment the received signals into distinct constituent structure components while identifying the modulation types corresponding to each segment simultaneously.
These three tasks depend on the extraction of multi-scale hierarchical features from in-phase and quadrature (IQ) sequences, spanning low-level signal characteristics to high-level semantic representations, to accomplish their respective identification objectives.
Specifically, the differences in feature extraction among the three tasks are as follows:
(1) AMR primarily performs coarse-grained analysis, classifying entire signal segments into specific modulation types.
(2) SSR similarly operates at a coarse-grained level but emphasizes long-term temporal features, as different signal schemes exhibit unique physical-layer signal structures.
(3) SSP requires fine-grained feature extraction, analyzing individual signal samples to jointly classify structural segments and modulation types. This approach not only facilitates signal identification but also enables the analysis of unknown signal structures.
Feature extraction from raw IQ sequences is critical for accurate blind signal analysis, directly impacting recognition and classification performance.

Recently, propelled by its monumental success across various fields, deep learning (DL) has found widespread and highly successful applications in the blind analysis of communication signals. A variety of prominent architectures, including convolutional neural networks (CNNs), recurrent neural networks (RNNs), long short-term memory (LSTM) networks, and Transformers, have been established as powerful feature extraction methods in this domain. In particular, Transformer-based methods have emerged as a dominant force. Building on their remarkable success in natural language processing and diverse scientific domains, such as solving ordinary differential equations \cite{becker2023predicting}, modeling physical dynamical systems \cite{Yang2023FLUIDGPTL}, and time-series forecasting \cite{nie2023patchtst}, \cite{liu2023itransformer}, Transformers have demonstrated a profound capability to capture complex, long-range temporal dependencies inherent in sequential communication data.
Despite these significant architectural advances, current research on DL-based blind signal analysis focuses mainly on developing novel network architectures, advanced loss functions, or optimization algorithms. Limited attention has been paid to the fundamental aspect of how to effectively exploit the inherent complex-valued structure of the raw IQ input data. 

\color{black}
Current research on the utilization of the complex-valued communication signal using DL can be classified into two main paradigms (as illustrated in Fig.~\ref{fig:paradigm}).
\begin{figure}[!t]
\centering
\subfloat[]{\includegraphics[width=0.26\linewidth]{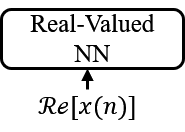}}
\hspace{1cm}
\subfloat[]{\includegraphics[width=0.36\linewidth]{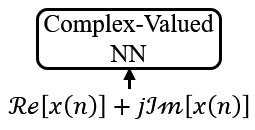}}
\caption{Comparison of real-valued and complex-valued paradigms.
(a) RVNN. 
(b) CVNN.
 }
\label{fig:paradigm}
\end{figure}
\textbf{(1) Real-Valued:} 
The real-valued approach can be implemented in two ways: using only the I component as  the input of the model or concatenating both the I and Q components to form the input signal, thus facilitating the application of conventional real-valued neural network (RVNN) architectures. 
In this class of methods, early approaches utilized CNNs for modulation classification \cite{o2016convolutional,zhang2024class,lin2020improved}. 
To capture temporal dependencies, RNNs, especially LSTM variants, were later adopted \cite{rajendran2018deep, ke2021real}. 
Subsequent hybrid CNN-RNN architectures further improved performance\cite{zhang2020automatic ,wang2021multidimensional}. 
Recently, transformer-based models have surpassed CNNs and RNNs in AMR by leveraging self-attention mechanisms. \cite{li2022transformer,zheng2022tmrn,chen2022abandon}.

\textbf{(2) Complex-Valued:} 
Since complex-valued signals inherently encapsulate amplitude and phase information, a recent trend is to treat the IQ signal as an inseparable entry and directly feed into the neural network. Complex-valued neural networks (CVNNs) enable direct feature extraction by processing complex numbers natively, employing complex weights, activation functions, and backpropagation algorithms.
%
Tu et al. \cite{tu2020complex} pioneered CVNNs for AMC, introducing core components such as complex convolution, batch normalization, and weight initialization. Subsequent work optimized CVNN architectures for reduced complexity \cite{xiao2023complex, luo2023complex}. Ren et al. \cite{ren2022complex} extended CVNNs to recurrent structures, proposing a parallel CNN-RNN model, while complex-valued LSTM and GRU variants were developed in \cite{yang2022automatic,xu2023novel}. Recently, Lei demonstrated that complex-valued Transformers outperform their real-valued counterparts by jointly modeling amplitude and phase through self-attention \cite{lei2024understanding, leng2025unveiling,LiDWYH24}.

Previous studies have demonstrated that employing IQ signals can effectively achieve a better theoretical performance limit; however, this approach also leads to increased computational complexity and greater challenges in model optimization.
Theoretical analysis and experimental results demonstrate that the generalization error, consisting of approximation error and estimation error, exhibits fundamentally different characteristics between CVNNs and RVNNs. Two key findings are observed. First, CVNNs achieve enhanced representational capacity through their higher-dimensional parameter space, leading to theoretically guaranteed smaller approximation error. Second, the resulting complex error surface morphology causes CVNNs to converge more frequently to spurious local minima, thereby producing greater estimation error compared to RVNNs.

In summary, a critical challenge in DL-based blind signal analysis research lies in two aspects:
\begin{enumerate}
    \item Regarding complex-valued data, numerous practices exist for adapting it to RVNNs. However, these methods often lack a solid theoretical foundation to rigorously justify their effectiveness. Moreover, as CVNNs are high-capacity models that typically require large amounts of training data to find good solutions, these adaptation approaches fail to address the fundamental challenge of data scarcity in non-cooperative scenarios.
    \item A key challenge lies in designing a model that effectively balances the trade-off between approximation error and estimation error. The central technical problem involves effectively leveraging I and Q components to maintain the representational benefits of complex-valued architectures while addressing their inherent optimization instability.
\end{enumerate}

In this paper, we explore a new alternative for processing complex-valued communication signals: Dual-channel Neural Networks (DualNN), where the network parameters for processing the real part and the imaginary part signals are shared. The general framework is accompanied by mathematical proofs on its properties. To demonstrate the benefits of DualNN, we further propose a dual-channel Transformer architecture (DualFormer) for specific tasks, incorporating two tailored projection headers designed to meet different granularity requirements: Header V1 for coarse-grained tasks (AMR and SSR) and Header V2 for fine-grained tasks (SSP).

Contributions of this paper are summarized as follows.
\begin{enumerate}
\item Propose DualNN, a groundbreaking framework for processing complex-valued communication signals using RVNN. Our innovative symmetric architecture with parameter sharing not only stabilizes training but also dramatically reduces model complexity.
\item Provide mathematical proof, rooted in generalization theory, that the proposed DualNN achieves a tighter generalization error bound than RVNN universally and outperforms CVNN under limited training data. Unlike RVNN, which fails to capture complex-valued structures, DualNN leverages shared transformation to extract and combine information from the I and Q components of complex-valued signals. While CVNN requires abundant data to learn these relationships, the parameter-sharing architecture of the proposed DualNN enables robust performance in data-scarce scenarios. 
\item Propose Dualformer, a novel dual-channel Transformer module that processes signal patches as tokens while effectively exploiting I and Q component relationships. Specifically, we design two specialized projection headers to achieve multi-scale feature extraction and ensure high-precision performance across different tasks including AMR, SSR, and SSP. Simulation results demonstrate the proposed Dualformer outperforms the state-of-the-art methods in these three tasks.

\end{enumerate}

The rest of the article consists of five parts.
Section II illustrates the task and model of blind communication signal analysis.
Section III introduces the dual-channel neural network architecture and gives a theoretical analysis of the generalization error to show the advantages of our proposed architecture.
Section IV presents Dualformer, a specific instantiation of DualNN that adopts the Transformer architecture as its core model.
Experiments on AMR, SSR, and SSP are conducted in Section V.
Section VI concludes the paper.

\section{Problem Statement}
\subsection{Model of Communication Signals}
A typical wireless communication system includes a transmitter, wireless channel, and receiver.  
The transmitted information goes through a source encoder, channel encoder, modulator, digital-to-analog (D/A) converter, and shaping filter to obtain the baseband signal ${x_b}\left( n \right)$, expressed as 
\begin{equation}
\begin{array}{l}
{x_b}\left( n \right) = \underbrace {\sum\limits_{m = 0}^{{M_1} - 1} {b_m^{{k_1}}} g\left( {n - m{T_{{b_1}}}} \right)}_{a_1} +  \cdots \\
\underbrace {\sum\limits_{m = 0}^{{M_S} - 1} {b_m^{{k_S}}} g\left( {n - m{T_{{b_S}}} - \sum\limits_{s = 1}^{S - 1} {{M_s}{T_{{b_s}}}} } \right)}_{a_S}
\end{array},
\end{equation}
where the subscript $s$ ($s \in \{{1,...,S}\}$) represents the section index and $S$ is the total number. 
The $s$-th section $a_s$ contains $M_s$ symbols and utilizes the $k_s$-th ($k_s \in \{{1,...,K}\}$) modulation type with a symbol period at $T_{b_s}$. $K$ represents the total number of candidate modulation types.
$b_m^{k_s}$ denotes the amplitude, frequency, and phase for the $m$-th symbol in the $s$-th section, while $g\left(n\right)$ represents the shaping filter.

The baseband signal $x_b(n)$ is then shifted by the carrier frequency to obtain the bandpass signal $x_c\left(t\right)$, which can be simplified as
\begin{equation}
\begin{gathered}
  x_c\left( n \right) = \mathcal{R} \left\{ {{x_b}\left( n \right)} \right\}\cos \left( {2\pi {f_c}n} \right) - j\mathcal{I}\left\{ {{x_b}\left( n \right)} \right\}\sin \left( {2\pi {f_c}n} \right) \hfill \\
  \quad \;\;\; = {x_{real}}\left( n \right)\cos \left( {2\pi {f_c}n} \right) - j{x_{imag}}\left( n \right)\sin \left( {2\pi {f_c}n} \right) \hfill \\ 
\end{gathered} ,
\end{equation}
where ${x_{real}}\left(n\right) = \mathcal{R}\{{{x_b}\left(n\right)}\}$ and ${x_{imag}}\left(n\right) = \mathcal{I}\{{{x_b}\left(n\right)}\}$ represent the in-phase and quadrature components, respectively.
In practical wireless communication systems, the transmitted signal is affected by various imperfections in channels and receivers. These include phase noise, frequency shift, and additive white Gaussian noise. 
Furthermore, inaccurate detection introduces residual noise at the beginning and end of the received signal. Consequently, the received signal can be written as
\begin{equation}
x\left( n \right) = R\left[ {H\left[ {x_c\left( n \right)} \right]} \right],
\label{eq_sig}
\end{equation}
where $H\left[\cdot\right]$ represents the channel effects and $R\left[\cdot\right]$ denotes the receiving process.

\subsection{Tasks of Blind Communication Signal Analysis}
Blind communication signal analysis tasks can be divided into coarse-level classification of entire signals (e.g., AMR and SSR) and fine-grained labeling of individual sampling points (e.g., SSP).

\subsubsection{AMR}
In AMR, signals are typically modeled with a single modulation type, and the task aims to identify this modulation from intercepted signals \cite{wang2025survey}. 
Formally, the task can be expressed as
\begin{equation}
    {\cal F} :  x(n)  \to  k,
    \label{eq_amr}
\end{equation}
where $k$ refers to the $k$-th modulation type.

\subsubsection{SSR}
Similar to AMR, SSR aims to identify the signal scheme from intercepted signals, which can be formally expressed as 
\begin{equation}
    {\cal F} :  x(n)  \to  {\cal X}_,
    \label{eq_amr}
\end{equation}
where ${\cal X}$ denotes the  signal scheme, which corresponds to a uniquely structured signal comprising a preamble and data section with fixed lengths and distinct modulation characteristics.

\subsubsection{SSP}
The signal structure parsing task aims to decompose the received signal by identifying its constituent sections and their respective modulation types. 
Let $f_s$ denote the sampling rate and $N_L$ represent the total number of sampling points in the signal, given by
\begin{equation}
    N_L=\sum_{s=1}^{S}{N_s},
\end{equation}
where $N_s = f_s \times T_{b_s} \times M_s$ denotes the number of sampling points in $s$-th section.
We can model the signal structure parsing task as a supervised time series classification problem with the composite labels $(a(n), k(n))$, expressed as
\begin{equation}
{\cal F} :  x(n)  \to (a(n), k(n)) , n=1,2,\cdots,N_L,
\end{equation}
where $a(n)$ and $k(n)$ represent the section and modulation type label for the $n$-th sampling point, respectively.

\section{Dual-Channel Neural Network and Generalization Error Analysis}
\subsection{Dual-Channel Neural Network}
We propose a dual-channel neural network framework, hereafter referred to as DualNN, specifically designed for blind analysis of complex-valued communication signals. 
The overall workflow and architectural details of DualNN are visually represented in Fig.~\ref{fig:dualnn}.
\begin{figure}[!t]
    \centering
    \includegraphics[width=0.85\linewidth]{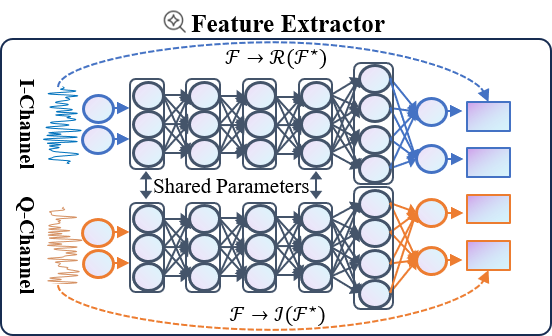}
    \caption{Architecture of DualNN for communication signal analysis.}
    \label{fig:dualnn}
\end{figure}

The core design of our proposed framework lies in its innovative approach to handling I and Q data. Recognizing that complex-valued communication signals are inherently represented by I and Q components, our framework is uniquely structured to process both I and Q inputs using shared parameters. 
More specifically, the feature extraction process of DualNN can be written as follows:
\begin{equation}
    \mathbf{z}^{(I)} = \mathcal{F}_{dual}(\mathbf{x}_{real}),\ 
    \mathbf{z}^{(Q)} = \mathcal{F}_{dual}(\mathbf{x}_{imag}),
\end{equation}
where $\mathcal{F}_{dual}$ is a RVNN and $\mathbf{x}_{real}$ and $\mathbf{x}_{imag}$ are the real and imaginary parts of the input signal, respectively. To see the difference from CVNN, its corresponding feature extraction process is written as
\begin{equation}
    \mathbf{z} = \mathcal{F}_{comp}(\mathbf{x}_{real}+j\mathbf{x}_{imag}),
\end{equation}
where $\mathcal{F}_{comp} = {\mathcal{R}}({\cal F}_{comp}) + j{\mathcal{I}}({\mathcal{F}}_{comp})$.
Compared to the CVNN, the proposed DualNN shows two advantages. \color{black} Firstly, it introduces a physically justified inductive bias that explicitly respects the intrinsic statistical properties of the I and Q channels. Many standard complex baseband communication signals (e.g., M-ary QAM and PSK) possess circular symmetry, meaning that for any phase rotation $\phi$, $x(n)$ and $e^{j\phi}x(n)$ follow the same distribution. From a second-order perspective, this circularity implies a vanishing pseudo-covariance $\mathbb{E}[x(n)^2]=0$, which for zero-mean signals expands to $\mathbb{E}[x_{real}^2(n)]=\mathbb{E}[x^2_{imag}(n)]$ and $\mathbb{E}[x_{real}(n)x_{imag}(n)]=0$. Hence, I and Q share identical second-order statistics. Moreover, under stationary or approximately locally stationary conditions, the temporal statistics remain stable over time, allowing a consistent feature extraction rule to be applied across signal segments. By sharing parameters between the two branches rather than processing them through independent or unconstrained networks, DualNN explicitly leverages this underlying physical symmetry and stationarity, \color{black} thereby leading to a robust representation of the complex-valued signal.
Secondly, with identical network depth and parameters for processing the real and imaginary parts of the signal, the proposed DualNN framework reduces computational complexity by half compared to the CVNNs. This efficiency gain is achieved through parameter sharing, which additionally enhances training convergence and helps avoid poor local optima.

In general, the proposed DualNN framework is compatible to any neural network modules depending on different task requirements. In Section IV, we will demonstrate how the DualNN is adopted to the three blind signal analysis tasks described in Section II.

\subsection{Generalization Error}
Blind signal analysis is to map complex-valued signals to real-valued labels $y$ via the function ${\cal F}$, expressed as
\begin{equation}
    {\cal F}(\mathbf{x}) = \sigma_L(\mathbf{A}_L \sigma_{L-1}(\mathbf{A}_{L-1} \cdots \sigma_1(\mathbf{A}_1 \mathbf{x}))),
    \label{eq_fa}
\end{equation}
where $\mathbf{A}_l$ and $\sigma_l$ are the weight matrix and activation function for the $l$-th layer. 
The network output ${\cal F}(\mathbf{x}) \in \mathbb{R}^{d_{L}}$ (with $d_{0} = N_L$ and $d_{L} = K$) is converted to a class label in $\{1, \ldots, K\}$ by taking the $\arg\max$ over components.

Suppose there are a total of $N$ instances, the empirical risk minimization can be written as 
\begin{equation}
    {\cal E}({\cal F}) = \frac{1}{N} \sum_{n=1}^{N} {\cal L}({\cal F}(\mathbf{x}_n),y_n),
\end{equation}
where ${\cal L}({\cal F}(\mathbf{x}_n),y_n)$ is the loss function measuring the error between the predicted result ${\cal F}(\mathbf{x}_n)$ and the ground truth $y_n$ for the $n$-th instance. 

\begin{definition}[Generalization Error]
The performance of a neural network architecture comprises the approximation error and the estimation error \cite{jakubovitz2019generalization}, written as:
\begin{equation}
\underbrace {|{\cal E}({{\cal F}^ \star }) - {\cal E}(\widehat {\cal F})|}_{{\varepsilon^{general}}} = \underbrace {|{\cal E}({{\cal F}^ \star }) - {\cal E}(\widehat {{{\cal F}^ \star }})|}_{{\varepsilon^{approx}}} + \underbrace {|{\cal E}(\widehat {{F^ \star }}) - {\cal E}(\widehat {\cal F})|}_{{\varepsilon^{est}}},
\end{equation}
where  
${\cal F}^{\star}$ represents the optimal function, $\hat{\cal F}^{\star}$ is the best classifier that a neural network architecture can learn within its hypothesis space $\cal H$, and $\hat{\cal F}$ is the final classifier that is learnt based on a given training dataset ${\cal D}=\{x_n,y_n|n=1,2,\cdots, N\}$.
$|{\cal E}({\cal F^{\star}})-{\cal E}(\hat{{\cal F}^{\star}})|$ is the approximation error, which refers to the distance between the target function and the closest neural network function of a given architecture.
$|{\cal E}(\hat{{\cal F^{\star}}})-{\cal E}(\hat{{\cal F}})|$ is the estimation error, which refers to the distance between this ideal network function and an estimated network function.
\end{definition}

Next, we are going to compare the generalization of RVNN, CVNN, and DualNN in the process of feature extraction of complex-valued signals.
For RVNNs, we assume that ${\cal F}_{real}$ uses $L$ weight matrices $\mathbf{A}_l \in \mathbb{R}^{d_{l-1}\times d_{l}}$ and the fixed nonlinearities $\sigma_l \in \mathbb{R}^{d_{l}}$ is $\rho_l$-Lipschitz.
Analogously, weight matrices ${\mathbf{A}}_l \in \mathbb{C}^{d_{l-1}\times d_{l}}$ and nonlinearities $\sigma_l \in \mathbb{C}^{d_{l}}$ are preset for CVNNs ${\cal F}_{comp}$.

We first present an intuitive comparative analysis of the generalization errors among the three methods, as visually demonstrated in Fig.~\ref{fig:hypospace}. 
The hypothesis space $\mathcal{H}$, which represents the set of candidate functions used to approximate target mappings, serves as a critical framework for this comparison. To ensure fair comparison, we impose the constraint that all three methods maintain an identical number of neural nodes. Under this configuration, if the hypothesis space capacity for both RVNN and DualNN is quantified as $k$, CVNN exhibits a doubled capacity of $2k$. 
The nested family of hypothesis spaces ${\cal H}^{(k)}$ can be written as
\begin{equation}
    {\cal H}^{(0)} \subset {\cal H}^{(1)} \subset \cdots \subset {{\cal H}^{(k)}} \subset \cdots
    \subset {{\cal H}^{(2k)}} \subset \cdots,
\end{equation}
where ${\cal H}^{(k)}$ represents a sequence of hypothesis spaces and the hypothesis space is represented by ${\cal H}=\cup_k{\cal H}^{(k)}$.

\begin{figure}[!t]
    \centering
    \includegraphics[width=0.85\linewidth]{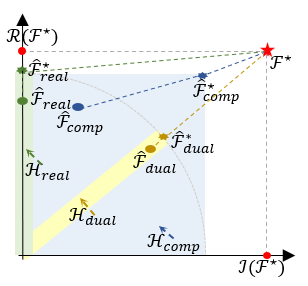}
    \caption{Visualization of generalization error for real-valued, complex-valued, and dual-channel neural networks.}
    \label{fig:hypospace}
\end{figure}

Fig. \ref{fig:hypospace} illustrates a critical trade-off: although CVNNs' expanded hypothesis space reduces approximation error through enhanced representational power, it concurrently amplifies estimation error due to increased model complexity under finite training instances. 
This phenomenon aligns with the bias-variance dilemma in statistical learning. DualNNs, however, strategically balance these competing factors. By constraining hypothesis space dimensionality while preserving cross-component interactions, they mitigate overfitting risks without sacrificing critical representational flexibility. Consequently, DualNNs achieve a comprehensive performance advantage in generalization error compared to RVNNs and CVNNs.

It is also noteworthy that our analysis assumes the fully connected neural network (FNN), which, by virtue of the universal approximation theorem \cite{stinchcomb1989multilayered}, dominates the entire hypothesis space. This theoretical framework extends to other prominent architectures, including CNNs \cite{Jiang2021}, RNNs \cite{Li2021,Li2022}, and Transformer models \cite{jiang2024approximation}.

\subsection{Approximation Error}
For analytical convenience, we assume the cost function $\mathcal{L}$ in (11) takes a quadratic form, specifically the mean squared error (MSE), so that the empirical risk minimization (11) becomes \cite{girosi1995approximation}: 
\begin{equation}
    {\cal E}({\cal F}) = ||{\cal F}^{\star}-{\cal F}|| + {\cal E}({\cal F}^{\star}),
\end{equation}
and therefore the approximation error can be characterized as Definition 2 below.
The conclusions of subsequent proofs in this paper can be extended to other cost functions.
\begin{definition}[Approximation Error]
For a hypothesis space $\mathcal{H}$, the approximation error of $\mathcal{H}$ with respect to the target function $\mathcal{F}^{\star}$ is defined as
\begin{equation}
    \epsilon^{\text{approx}}(\mathcal{H}) = \inf_{f \in \mathcal{H}} \bigl\| \mathcal{F}^{\star} - f \bigr\|.
\end{equation}
In our specific setting, we measure the accuracy of an approximation $\mathcal{F}_k(\mathbf{x})$ within the hypothesis space $\mathcal{H}^{(k)}$ to the target function $\mathcal{F}^{\star}(\mathbf{x})$ in terms of the $l_2(\mu,B)$ norm,
\begin{equation}
    \bigl\| \mathcal{F}^{\star} - \hat{\mathcal{F}}_k^{\star} \bigr\| = \int_B \bigl| \mathcal{F}^{\star}(\mathbf{x}) - \hat{\mathcal{F}}_k^{\star}(\mathbf{x}) \bigr| \, \mu(d\mathbf{x}),
\end{equation}
for an arbitrary probability measure $\mu$ with domain $B$. The integral value quantifies the discrepancy between the optimal function $\mathcal{F}^{\star}$ and the best approximated function $\hat{\mathcal{F}}_k^{\star}$ within $\mathcal{H}^{(k)}$.
\end{definition}

Definition 2 quantifies the approximation error, and Lemma 1 provides its upper bound.

\begin{lemma}[Jackson-type Approximation Rate]
The Jackson-type Approximation Rate \cite{barron1994approximation} quantifies how well a target function ${\cal F}^{\star}(\mathbf{x})$ can be approximated by functions within a specific hypothesis space, particularly as the complexity or budget of these approximation spaces $k$ increases

\begin{equation}
\mathop {\inf }\limits_{{\hat {\cal F}^{\star}}} \left\| {{\cal F}^{\star} - \hat {\cal F}^{\star}} \right\| \le E\left( {{C^{(\alpha )}}({\cal F^{\star}}),k} \right).
\end{equation}
It states that for any function belonging to a designated target space ${C}^{(\alpha)}$ (characterized by a finite complexity measure $C^{(\alpha)}({\cal F}^{\star})$), the minimum approximation error, $||{\cal F}^{\star}-\hat{\cal F}^{\star}||$, is bounded by an error function $E(C^{(\alpha)}({\cal F}^{\star}), k)$. Futhermore, this error bound $E(\cdot,k)$ decreases to zero as $k$ approaches infinity, with the rate of this decay being termed the approximation rate. Essentially, it provides a theoretical guarantee on the achievable accuracy, linking the inherent complexity of the target function to the efficiency of its approximation as the model's capacity grows.
\end{lemma}

Based on the error bound (15), we can prove the following Theorem 1.
\begin{theorem}[Comparison of Approximation Error]
\label{conclusion_approx}
With identical network depth and number of parameters for processing the real and imaginary parts of the signal, the supremum of the approximation error for CVNN, RVNN, and DualNN satisfies 
\begin{equation}
  \sup  \epsilon_{comp}^{approx} \le \sup \epsilon_{dual}^{approx} \le
   \sup \epsilon_{real}^{approx} .
\end{equation}
\end{theorem}
\begin{proof}
See Appendix A.
\end{proof}

\subsection{Estimation Error}
\begin{definition}[Estimation Error]
According to the analysis in \cite{barron1994approximation}, the estimation error for the neural network can be expressed as 
\begin{equation}
   \epsilon^{est} =  \underset{(\mathbf{x}, y) \sim \mathcal{D}}{\mathbb{E}}[{\cal L}(\mathcal{F}(\mathbf{x}), y)] -  \frac{1}{N} \sum_{i=1}^{N} {\cal L}(\mathcal{F}(\mathbf{x_i}), y_i).
   \label{eq_geb_real}
\end{equation}
\end{definition}
The upper bounds of the error (\ref{eq_geb_real}) are architecture-dependent, varying with different network designs. We first present the bound for RVNN through Lemma 2.

\begin{lemma}[Estimation Error for RVNN]
In the real-valued case \cite{barron1994approximation, Bartlett2017}, for any \( \delta \), with probability at least \( 1 - \delta \), 
\begin{equation}
\begin{array}{l}
 \epsilon_{real}^{est }
 \le \frac{{8M}}{{{N^{\frac{3}{2}}}}} + \frac{{36{{\left\| {\bf{x}} \right\|}_2}\sqrt {2\ln (\sqrt2 W)} \ln (N){R_{\cal A}^{real}}}}{N} + 3M\sqrt {\frac{{\ln \frac{1}{\delta }}}{{2N}}} 
\end{array} ,
\label{eq_geb_dual}
\end{equation}
where $N$ is the amount of instances,
$M$ represents the upper boundary for the loss function, i.e., ${\cal L}(\mathcal{F}(\mathbf{x}),y)\le M$ for any $(\mathbf{x},y)$.
$W$ stands for the largest size of hidden layers, i.e., the maximum of $\{d_1,d_2,\cdots,d_L\}$.
$R_{\cal A}^{real}$ is the spectral complexity of the RVNN $\mathcal{F}_{real}$ (details presented in Appendix B).
\end{lemma}

Similarly to RVNN, the estimation error bounds of CVNN and DualNN follow the same expression as given by Lemma 3 and Corollary 1 respectively, differing only in their model-specific spectral complexity terms.

\begin{lemma}[Estimation Error for CVNN]
The estimation error boundary for the CVNN \cite{chen2023spectral} can be expressed as
\begin{equation}
\begin{array}{l}
 \epsilon_{comp}^{est}  \le \frac{{8M}}{{{N^{\frac{3}{2}}}}} + \frac{{36{{\left\| {\bf{x}} \right\|}_2}\sqrt {2\ln (2 W)} \ln (N){R_{\cal A}^{comp}}}}{N} + 3M\sqrt {\frac{{\ln \frac{2}{\delta }}}{{2N}}} 
\end{array} .
\label{eq_geb_comp}
\end{equation}
where $R_{\cal A}^{comp}$ is the spectral complexity of the CVNN $\mathcal{F}_{comp}$ .
\end{lemma}

\begin{corollary}[Estimation Error for DualNN]
The estimation error boundary for the DualNN can be expressed as
\begin{equation}
\begin{array}{l}
 \epsilon_{dual}^{est}  \le \frac{{8M}}{{{N^{\frac{3}{2}}}}} + \frac{{36{{\left\| {\bf{x}} \right\|}_2}\sqrt {2\ln (\sqrt2 W)} \ln (N){R_{\cal A}^{dual}}}}{N} + 3M\sqrt {\frac{{\ln \frac{1}{\delta }}}{{2N}}} 
\end{array} ,
\label{eq_geb_dual}
\end{equation}
where $R_{\cal A}^{dual}$ is the spectral complexity of the DualNN $\mathcal{F}_{dual}$.
\end{corollary}

Based on the error bound (20) and (21), the following inequality is satisfied in the general case:
\begin{equation}
   \sup{\epsilon_{real}^{est}}<\sup {\epsilon_{comp}^{est }},
\end{equation}
Then we can prove the following Theorem 2.

\begin{theorem}[Estimation Error Comparison] 
\label{conclusion_est}
If the cost function $\cal L$ in (11) takes a quadratic form,  estimation errors of CVNN and DualNN exhibit the following relationship:
\begin{equation}
   \sup{\epsilon_{dual}^{est}}<\sup {\epsilon_{comp}^{est } }.
\end{equation}
\end{theorem}

\begin{proof}
See Appendix B.
\end{proof}
\color{black}
From an optimization perspective, a tighter estimation bound on the spectral complexity $R_{\mathcal{A}}$ is directly linked to improved training stability and generalization performance. A large maximum eigenvalue (spectral norm) of the weight matrices amplifies small input perturbations across layers, thereby producing a loss surface characterized by sharp narrow valleys and steep ridges. Such a rugged landscape tends to trap gradient-based optimizers in spurious local minima, leading to high estimation error. In the proposed DualNN, the shared weight matrix $\mathbf{A}_i$ is jointly optimized to minimize the reconstruction error on both the I and Q channels. This joint objective imposes a mutual constraint: the learned transformation must perform well on two statistically identical yet distinct input streams. Consequently, the maximum eigenvalue is significantly restricted, such that $\lambda_{\text{dual}} < \lambda_{\text{comp}}$ (details in Appendix B). The practical effect is twofold. First, the reduced spectral norm at each layer lowers the global Lipschitz constant of the network, making the loss function smoother with respect to parameter variations. Second, a flatter loss surface offers wider basins of attraction around favorable minima, enabling gradient descent to converge to solutions with strong generalization ability rather than sharp, overfitted minima. Hence, the observed reduction in estimation error is directly attributable to the stabilized training induced by the spectral norm constraints inherent to DualNN.

\subsection{Summary}
Theoretical analysis above demonstrates that RVNN exhibits greater approximation error than both CVNN and DualNN, since the error contribution from the imaginary component of the optimal model cannot be fully eliminated. 
While RVNN demonstrates reduced estimation errors, CVNN and DualNN exhibit monotonically decreasing approximation errors with growing model complexity. The incremental benefits of RVNN's estimation accuracy are insufficient to offset its approximation error disadvantage relative to CVNN and DualNN, consequently yielding superior generalization performance for the latter architectures.
 
Although DualNN might not be able to approximate the true function as perfectly as CVNN, it attains smaller estimation errors \color{black} by mathematically guaranteeing a more stable optimization landscape. \color{black} When the model scale becomes sufficiently large, both DualNN and CVNN demonstrate comparably small approximation errors. Under these conditions, DualNN achieves a statistically significant reduction in overall generalization error.
This will be demonstrated through experimental results in Section V. Overall, DualNN achieves better generalization performance relative to CVNN and RVNN, as visually demonstrated in Fig.~\ref{fig:hypospace}.

\section{Dualformer: A Transformer-Based Implementation of DualNN}
The Transformer architecture offers inherent advantages over classical deep learning models (CNNs, RNNs), particularly in capturing long-range dependencies and facilitating parallel computations. Therefore, we adopt it as the core of our neural network and propose a specific implementation of DualNN: the Dual-Channel Transformer (Dualformer) method, characterized by its dual-channel input and shared Transformer model parameters.

The schematic illustration of Dualformer in blind communication signal analysis is presented in Fig.~\ref{fig:framework}. The framework adopts a sequential processing pipeline, comprising the following stages:
(1) IQ data instances are segmented into patches, which serve as tokens for subsequent processing (Section IV-A);
(2) Semantic features are extracted from each patch using a dual-channel Transformer architecture (Section IV-B);
(3) The extracted features are further processed by a task-specific projection head, adapted for AMR, SSR, and SSP tasks, respetively (Section IV-C).

Recent studies show Transformers are increasingly adopted for blind signal analysis due to their superior performance over conventional CNNs, RNNs, and hybrid approaches. The architecture's strengths include modeling long-range temporal dependencies, processing large-scale datasets, and extracting high-level features through contextual attention, enabled by its parallelizable self-attention mechanism that facilitates cross-domain feature integration, particularly advantageous in complex communication environments. Based on these advantages, we employ Transformer as our backbone architecture.

\begin{figure*}
    \centering
    \includegraphics[width=0.95\linewidth]{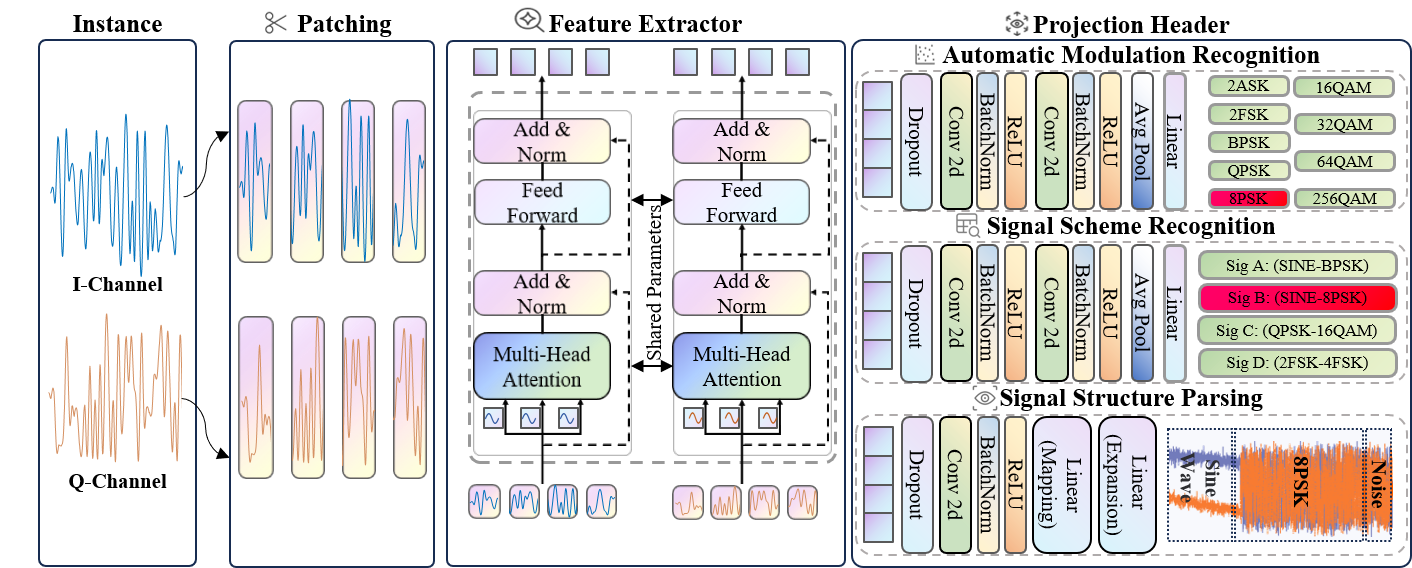}
    \caption{Pipeline of Dualformer-based 
    feature extraction in blind communication signal analysis.}
    \label{fig:framework}
\end{figure*}

\subsection{Patching as the Tokenizer}
We denote a complex-valued communication signal of length \( N_L \) starting at time index $1$ as \( \mathbf{x}_{1:N_L} = \mathbf{x}^{(I)}_{1:N_L} + j\mathbf{x}^{(Q)}_{1:N_L} = (x^{(I)}_1, \ldots, x^{(I)}_{N_L}) + j(x^{(Q)}_1, \ldots, x^{(Q)}_{N_L}) \). 
Each input communication signal instance $\mathbf{x} \in \mathbb{C}^{1 \times N_L}$ is first divided into patches with patch length \( P \), resulting in a sequence of patches \( \mathbf{x}_p \in \mathbb{C}^{P \times N_P} \) where \( N_P = \left\lfloor (N_L-P)/P \right\rfloor + 2 \) is the number of patches. 
With the use of patches, the number of input tokens reduces from \( N_L \) to approximately \( N_L / P \). 
This implies the memory usage and computational complexity of the attention map are quadratically reduced by a factor of \( P \). 
Under constrained training time and GPU memory, patch design allows the model to see longer historical sequence, which can significantly improve feature extraction capability for the communication signals.
\color{black}
It is worth noting that the choice of patch length $P$ introduces a fundamental trade-off. While a larger $P$ quadratically reduces computational complexity, it introduces coarse tokenization that could degrade boundary localization, which is particularly detrimental for fine-grained tasks like SSP. In this work, $P=16$ is selected as a heuristic sweet spot that balances sequence compression with temporal resolution. Furthermore, despite this coarse tokenization, the global self-attention mechanism effectively aggregates contextual cues across the entire sequence. This provides sufficient in-context information for the subsequent projection headers, thereby enabling them to accurately infer and localize high-resolution transition boundaries between signal segments.

\color{black}
\subsection{Architecture of Dualformer}
The two patches \( {\mathbf{x}_p}^{(I)} \in \mathbb{R}^{P \times N_P} \) and \( {\mathbf{x}_p}^{(Q)} \in \mathbb{R}^{P \times N_P} \), are fed into the Transformer backbone.
The Transformer encoder processes input patches through four key steps: Firstly, patches are projected into a $D$-dimensional latent space using trainable weights $\mathbf{W}_p \in \mathbb{R}^{D \times P}$.
Secondly, position encodings $\mathbf{W}_{pos} \in \mathbb{R}^{D \times N_P}$ are added to preserve temporal ordering:
    \begin{equation}
        \mathbf{x}_d^{(I)} = \mathbf{W}_p \mathbf{x}_p^{(I)} + \mathbf{W}_{pos}, \quad
        \mathbf{x}_d^{(Q)} = \mathbf{W}_p \mathbf{x}_p^{(Q)} + \mathbf{W}_{pos},
    \end{equation}
    where $\mathbf{x}_d^{(I)}, \mathbf{x}_d^{(Q)} \in \mathbb{R}^{D \times N_P}$ are the encoder outputs. 
Thirdly, multi-head attention transforms these into query/key/value matrices for each head.
Finally, a scaled dot-product is used to get the attention output.
The multi-head attention block also includes BatchNorm layers and a feed-forward network with residual connections. 
Afterwards it generates the representation denoted as \( \mathbf{z}^{(I)} \in \mathbb{R}^{D \times N_P} \) and \( \mathbf{z}^{(Q)} \in \mathbb{R}^{D \times N_P} \). 
After being processed by the Transformer encoders, the feature representations are concatenated as $\mathbf{z} = [\mathbf{z}^{(I)}; \mathbf{z}^{(Q)}]$ and fed into the projection header.

\subsection{Projection Header}
\subsubsection{Basic Module}
The projection header module performs task-specific feature mapping by employing a sequential architecture comprising two-dimensional (2D) convolutional layers, batch normalization (BN), rectified linear unit (ReLU) activation, pooling operations, and fully connected (FC) layers.
Taking a single layer as an example, we employ  convolution expressed as follows.
\begin{equation}
\begin{split}
\mathbf{z}_{out} &= \text{2DConv}\left(\mathbf{z}_{in}, [\mathbf{W}_{conv}, \mathbf{b}_{conv}]\right) \\
&= \mathbf{W}_{\text{conv}} \otimes_{\text{2D}} \mathbf{z} + b_{\text{conv}},
\end{split}
\end{equation}
where $\mathbf{W}_{conv}$ is the weights of the convolution kernel and $\mathbf{b}_{conv}$ is the bias vector, $\otimes_{\text{2D}}$ represents 2D convolution operation.

The BN layer can be written as
\begin{equation}
    \mathbf{z}_{bn} = \text{BN}\left(\mathbf{z}_{out}\right) = 
      \frac{\mathbf{z}_{out}-\mu}{\sqrt{\sigma^2+\epsilon}},
\end{equation}
where $\mu$ and $\sigma$ are the mean and standard variance of the input feature $\mathbf{z}_{out}$. $\epsilon$ denotes a small constant for numerical stability.

The ReLU layer can be written as
\begin{equation}
    \mathbf{z}_{relu} = \text{ReLU}\left(\mathbf{z}_{bn} \right) = \text{max} \left( \gamma \textbf{z}_{bn}, 0 \right),
\end{equation}
where $\gamma$ represents a learnable scaling parameter.

The AvgPool layer can be written as
\begin{equation}
\mathbf{z}_{avg} = \text{AVGPOOL}\left(\textbf{z}_{bf}\right) = \left({\frac{1}{N_{avg}^2}\sum\limits_{m=1}^{N_{avg}}\sum\limits_{n=1}^{N_{avg}}\textbf{z}_{bn}}\right),
\end{equation}
where $N_{avg}$ is the size of the AvgPool layer.

The FC layer outputs the prediction probability, which can be written as
\begin{equation}
    \textbf{p} = \text{FC}\left(\mathbf{z}_{avg} \right) = \textbf{z}_{avg}\textbf{W}_{fc}+\textbf{b}_{fc},
\end{equation}
where $\mathbf{W}_{fc}$ and $\mathbf{b}_{fc}$ are the weight and bias of the FC layer.
As different tasks emphasize specific feature representations and learning objectives, we implement specialized projection headers through customized layer configurations, as detailed below.

\subsubsection{Header V1: Coarse Projection Header}
For the AMR and SSR tasks, we employ a two-layer 2D convolutional network, each layer comprising a 2D convolution followed by BN and ReLU activation, to hierarchically extract discriminative global features. AvgPool layer is subsequently applied to enhance feature invariance, while the dual-channel architecture ensures effective information fusion of IQ data. 
\color{black}
Specifically, because AMR and SSR are coarse-grained tasks aimed at identifying global signal attributes, the two-layer CNN combined with spatial average pooling is explicitly designed to expand the receptive field. This architectural choice achieves the higher level of semantic abstraction required for these tasks.
\color{black}
The entire procedure can be described as

\begin{multline}
\textbf{p} = \text{FC}\circ\text{AVGPOOL}\circ\text{ReLU}\circ\text{BN}\circ\text{2DConv}\\
    \circ\text{ReLU}\circ\text{BN}\circ\text{2DConv}(\mathbf{z}).
\end{multline}
    
A Softmax layer is utilized to finish the final predictions.
The process can be written as

\begin{equation}
    \hat p_i = \frac{e^{p_i}}{{\sum\limits_{i = 1}^{K} {e^{p_i}} }},
\end{equation}
where $p_i$ denotes the $i$-th element of the probability vector $\mathbf{p}$ output by the softmax layer. 
$K$ represents the number of modulation types or signal schemes. Finally, the index of the maximum value represents the predicted label. 

\subsubsection{Header V2: Fine-grained Projection Header}
 
For the SSP task, we employ a single 2D convolutional layer with BN layer for local feature extraction, followed by two FC layers to generate probability outputs for each sampling point. 
\color{black}
Unlike AMR and SSR, SSP is a fine-grained task requiring precise point-by-point classification. Therefore, we deliberately utilize a single-layer CNN and remove all spatial pooling operations. This design rationale preserves maximum temporal resolution, preventing the blurring of critical transition boundaries between the preamble and data sections. The process can be expressed as:
\color{black}
\begin{equation}
\textbf{p} = \text{FC}\circ\text{FC}\circ\text{ReLU}\\
    \circ\text{BN}\circ\text{2DConv}(\mathbf{z}).
\end{equation}

A Softmax layer is utilized to finish the final predictions.
The process can be written as
\begin{equation}
    {\hat p}_n^{\{k,l\}} = \frac{\exp{({p_n^{\{k,l\}}})}}{{\sum\limits_{\{k,l\} = \{1,1\}}^{KL} {\exp{(p_n^{\{k,l\}}})} }},
\end{equation}
where $L$ represents the number of sections. $KL$ represents the actual number of composite annotations.
Finally, the prediction result can be expressed as
\begin{equation}
    {({\hat k},{\hat l})}  = \arg \max \left( {\left[ {{{{{\hat p}}}_1},{{{{\hat p}}}_2}, \cdots ,{{{{\hat p}}}_{KL}}} \right]} \right),
\end{equation}
where ${({\hat k},{\hat l})}$ is the composite predicted label, containing the section information and the modulation type.

\section{Experiment Results and Analysis}

\subsection{Dataset Design}

\color{black}
\subsubsection{Simple AMR}
The first dataset is a synthetic AMR dataset generated through simulation, specifically designed to verify the theoretical aspects of the proposed method in a controlled environment. This dataset includes five common digital modulation schemes, i.e., BPSK, QPSK, 8PSK, 16QAM, and 64QAM.
Each signal instance consists of 128 sampling points, with a symbol rate of 2.4 kbps and a sampling rate of 9.6 kHz. 
The SNR is fixed at 12 dB. For each modulation type, 4000 samples are generated, resulting in a total of 20,000 samples.

\subsubsection{AMR}
The second dataset employed in our experiments is the RML 2018.01A benchmark published by DeepSig \cite{o2018over}. Although synthetically generated, this dataset is engineered to closely emulate real-world communication scenarios by incorporating a wide range of channel impairments, noise levels, and propagation effects. It serves as a rigorous testbed for evaluating the robustness and generalization capability of deep learning-based AMR methods.
This dataset contains 24 distinct modulation types. For each modulation class, there are 4096 signal instances, and each instance consists of 1024 sampling points. 
The SNR and channel conditions vary across the dataset to simulate diverse operational environments, including multipath fading, frequency offsets, and additive white Gaussian noise. 

\color{black}

\subsubsection{SSR}
We simulated diverse signal schemes, each comprising two distinct components: a preamble section and a data section. The preamble designed for synchronization incorporates five modulation types, including a 0.1s continuous wave, BPSK/QPSK/OQPSK with 30 repetitions of ``100'' bit patterns for burst interference robustness, and 2FSK with 12 repetitions of "10" patterns. The data section employs five core digital modulations (BPSK, QPSK, 8PSK, 16QAM, and 4FSK), with 4FSK exclusively paired with 2FSK preambles, yielding 17 unique signal combinations. All schemes operate at 2.4 kbps bit rate with 300 ms duration, sampled at 9.6 kHz (generating 500 samples per signal scheme). To emulate realistic channel conditions, we introduce impairments including: a $\pm 9.6$ Hz frequency offset ($\pm 0.001$ normalized), uniformly distributed phase offset $\theta \sim \mathcal{U}[0,2\pi)$ radians, additive white Gaussian noise (AWGN) with SNR $\gamma \in \{-4, 0, 4, 8, 12\}$ dB, and random timing uncertainties $\Delta t \leq 1$ ms to simulate detection inaccuracies.

\subsubsection{SSP}
The SSP experimental dataset shares identical signal types and combinatorial patterns with the SSR benchmark (17 schemes total). The key modification is the intentional insertion of variable duration noise segments ($t_{\text{noise}} \in [0,1]$ ms) both preceding and succeeding each signal instance to deliberately increase the complexity of parser.
In the SSP task, where the model must output both signal segmentation results and their corresponding modulation types, we employ a composite annotation format ``XX-YYYY''. Here, ``XX'' represents the section identifier (''HD'' for preamble section and ''DA'' for data section), while ''YYYY'' denotes the modulation type (e.g., ''BPSK'', ''QPSK''). Notably, each sampling point in the instance has a composite annotation.

In all experiments presented in this paper, the dataset was randomly partitioned into training, validation, and test sets with proportions of 0.8, 0.1, and 0.1, respectively.

\subsection{Experiment Configuration and Training Methods}

We randomly divided each dataset into training, validation, and test sets in a ratio of 8:1:1. 
The network is optimized according to the results of the training and validation sets.
Training is conducted without a predetermined number of epochs but continues until no further improvements is observed in the performance on the validation set. 
Maximum \color{black}
epochs 
\color{black}
are 200 and the early stopping threshold is set as 15.
Additionally, the patch length is set as 16.
To ensure a fair comparison, all experiments were conducted in a standardized environment. The hardware configuration utilized an AMD 16-core CPU (4.0 GHz) and an NVIDIA RTX 4090 GPU, operating on Windows 10.
The models were implemented using Python 3.8.0 and the PyTorch 2.1.2 framework.

\subsection{Evaluation Metrics}

To comprehensively assess the performance of the proposed algorithm, we employ accuracy and F1-score to evaluate AMR and SSR task, while utilizing recognition quality (RQ), segmentation quality (SQ), and panoptic quality (PQ) to evaluate SSP task.

\subsubsection{Accuracy}
Accuracy is calculated as the ratio of correct predictions to the total number of correct predictions \cite{hossin2015review}, which can be expressed as
\begin{equation}
Acc  = \frac{|TP|+|TN|}{|TP|+|TN|+|FN|+|FP|},
\end{equation}
where $|TP|$, $|TN|$, $|FN|$, and $|FP|$ respectively represent the number of instances of true positive (TP), true negative (TN), false negative (FN), and false positive (FP).
The optimal accuracy score is 1.00, and the minimal score is 0.00.

\subsubsection{$\text{F}1$-score}
The F1-score is the harmonic mean of precision and recall, providing a balance between these two metrics \cite{hossin2015review}. 
The F1-score is particularly useful when the class distribution is imbalanced, calculated as
\begin{equation}
F_1 = 2 \times \frac{Precision \times Recall}{Precision + Recall},
\end{equation}
where precision and recall are defined as,
\begin{equation}
Precision = \frac{|TP|}{|TP| + |FP|}, \quad Recall = \frac{|TP|}{|TP| + |FN|}.
\end{equation}
The F1-score ranges from 0 (worst) to 1 (best), with 1 indicating perfect precision and recall.

\subsubsection{Recognition Quality}
RQ is the fraction of correctly identified instances, expressed as:
\begin{equation}
    RQ = \frac{|TP|}{|TP|+\frac{1}{2}|FP|+\frac{1}{2}|FN|}.
    \label{eq:dq}
\end{equation}
The best RQ is 1.00, and the worst is 0.00.

\subsubsection{Segmentation Quality}
SQ measures how well the predicted segments align with their matched ground truth segments, calculated as 
\begin{equation}
    SQ =  \frac{\sum_{(p,g) \in TP}IoU(p,g)}{|TP|} , 
    \label{eq:sq}
\end{equation}
Intersection over Union (IoU) calculates the overlap rate between the ``predicted box'' and ``real box'', expressed as
\begin{equation}
    IoU = \frac{{\left| {A \cap B} \right|}}{{\left| {A \cup B} \right|}}, 
\end{equation}
where $A$ and $B$ are respectively the ``predicted box'' and the ``true box''. $\left| {A \cap B} \right|$ stands for their intersection and $\left| {A \cup B} \right|$ represents their union. 
The best case is a perfect overlap when the ratio $IoU = 1.00$.

\subsubsection{Panoptic Quality}
PQ is a comprehensive metric used to evaluate the performance of panoptic segmentation models \cite{kirillov2019panoptic}. 
The mathematical expression for PQ is defined as the multiplication of SQ and RQ, expressed as
\begin{equation}
    PQ = RQ \times SQ = \frac{{\sum\limits_{(p,g) \in TP} I oU(p,g)}}{{|TP| + \frac{1}{2}|FP| + \frac{1}{2}|FN|}}.
\end{equation}
PQ effectively combines the aspects of how many correct detections were made (RQ) and how accurate these detections were in terms of segmentation (SQ).

\color{black}
\subsection{Advantages of DualNN}
To verify the advantages of the proposed Dual structure, we avoid using Transformer architectures and instead implement DualNN with a feature extractor built from a 3-layer 1D convolution, capped with a classification head for the final recognition task. For fair comparison, both CVNN and RVNN adopt the same 3-layer 1D convolutional backbone and classification head, differing only in how they utilize the I/Q data. To eliminate interference from the complexity of dataset and isolate the structural benefits, we first conduct theoretical validation experiments using the simple AMR dataset. In subsequent sections, we sequentially validate our method from three perspectives: network structural advantages, sample efficiency, and the effect of I/Q imbalance.

\subsubsection{Network Structural Advantage}

To validate the structural advantage of DualNN, we conducted a controlled comparison among four networks: RVNN, CVNN, DualNN, and DualNN‑Comparable.
Due to its weight sharing mechanism, the standard DualNN with a hidden dimension of 128 has approximately half the number of parameters and FLOPs of CVNN. A direct comparison under such a disparity would be inherently unfair.
To enable a fair comparison with CVNN, we design DualNN‑Comparable. This variant enlarges the parameter count of DualNN to a level comparable to that of CVNN by appropriately increasing its hidden dimension. Its parameter count and FLOPs are thus roughly equivalent to those of CVNN.
Note that all parameter counts are reported in terms of real valued parameters, where each complex weight is counted as two real parameters. The network complexity of all four models is summarized in Table~\ref{table_param_match}.

A full grid search was performed over the hyperparameter space for each network to ensure a fair comparison. 
The search space comprised learning rate $\in \{1\times10^{-4}, 3\times10^{-4}, 1\times10^{-3}\}$, dropout rate $\in \{0.0, 0.1, 0.3\}$, and weight decay $\in \{0.0, 1\times10^{-4}, 1\times10^{-3}\}$, yielding 27 trials per model. All models were trained using the Adam optimizer with a batch size of 32. Training ran for a maximum of 200 epochs with early stopping applied on validation accuracy (patience = 15). The optimal hyperparameter configuration for each model was selected as the trial achieving the highest validation accuracy, and a fixed random seed of 1 was used throughout to ensure reproducibility. The selected optimal hyperparameters and corresponding test accuracies are presented in Table~\ref{tab:optimal_param}.
The results demonstrate that DualNN-Comparable achieved the highest test accuracy of 99.70\%, while DualNN attained 99.50\% ranking second among all models, despite having nearly half the parameters and FLOPs of both DualNN-Comparable and CVNN. The training dynamics of all four networks are illustrated in Fig.~\ref{fig:training_curves}. Among all models, the CVNN exhibits the slowest convergence, with its best checkpoint achieved as late as epoch 84, indicating greater optimization difficulty. 
This observation corroborates Theorem~2, confirming that the shared-weight constraint significantly restricts the maximum eigenvalue (spectral norm) of the weight matrices, yielding a flatter loss surface with wider basins of attraction and thereby enhancing optimization stability.

\begin{table}[!t]
\renewcommand{\arraystretch}{1.3}
\caption{Network Complexity}
\label{table_param_match}
\centering
\begin{tabular}{lccccr}
\hline
Model  & Params. & FLOPs.  \\
\hline
RVNN  & 92,165 & 2.12 M  \\
DualNN  & 92,805 & 4.25 M  \\
CVNN  & 183,690 & 8.50 M  \\
DualNN-Comparable  & 186,009 & 8.56 M \\
\hline
\end{tabular}
\end{table}

\begin{table*}[!t]
\renewcommand{\arraystretch}{1.3}
\caption{\textcolor{black}{Optimal Hyperparameter Configurations and Test Accuracies} }
\label{tab:optimal_param}
\centering
\begin{tabular}{lccccr}
\hline
Model & Best Learning Rate & Best Dropout & Best Weight Decay & Best Epoch & Test Accuracy \\
\hline
RVNN  & $1\times10^{-3}$ & 0.0 & 0.0& 23& 96.87\% \\
DualNN & $3 \times 10^{-4}$ & 0.0 & $1\times10^{-4}$& 46& 99.50\% \\
CVNN &  $1\times10^{-3}$ & 0.0 & $1\times10^{-4}$& 84& 99.35\% 
\\
DualNN-Comparable & $3 \times 10^{-4}$ & 0.0 & $1\times10^{-4}$&  60 & 99.70\% 
\\
\hline
\end{tabular}
\end{table*}

\begin{figure}
    \centering
    \subfloat[]{\includegraphics[width=\linewidth]{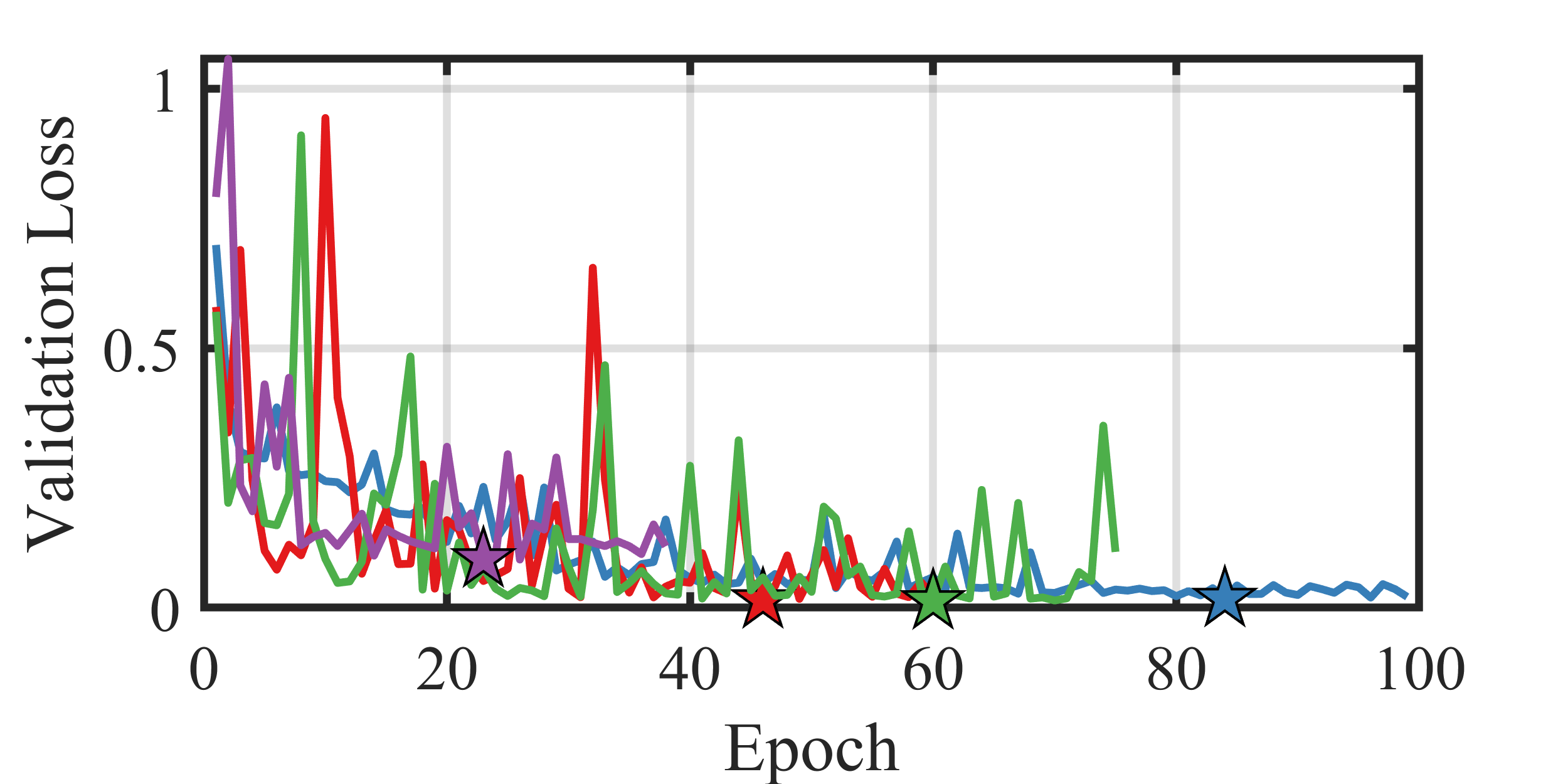}}
    \\
    \subfloat[]{\includegraphics[width=\linewidth]{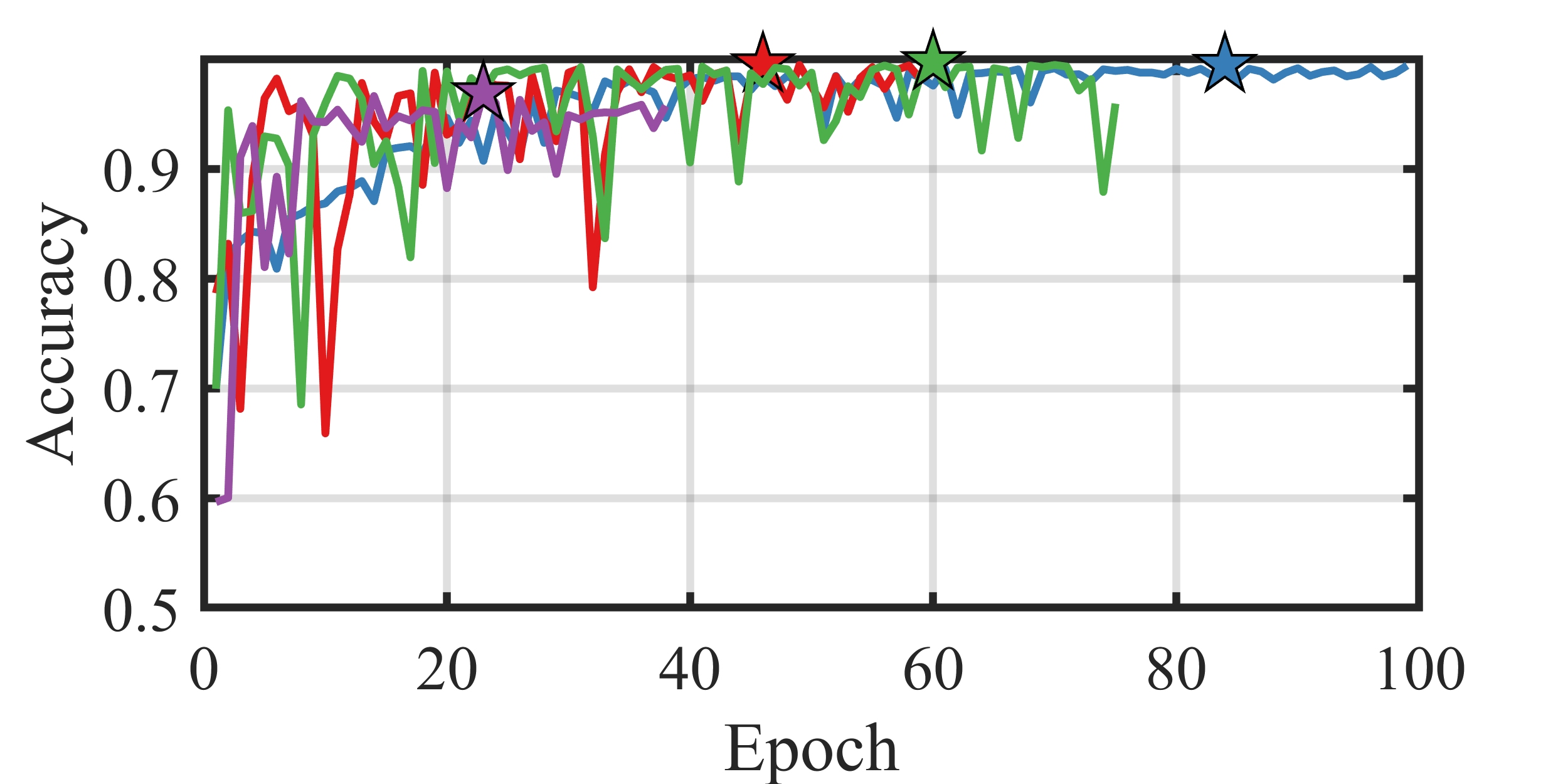}}
    \\
    \includegraphics[width=0.9\linewidth]{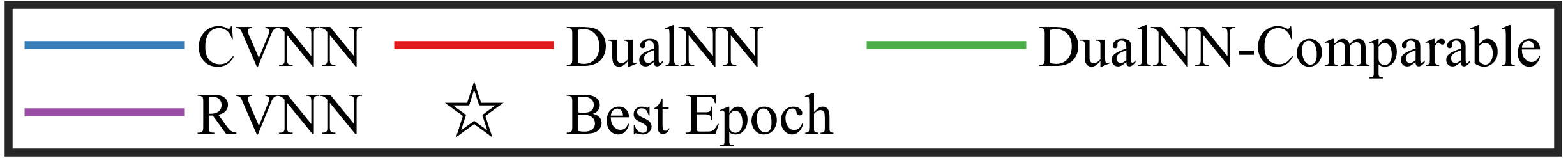}
    \caption{\textcolor{black}{Training dynamics of RVNN, CVNN, DualNN, and DualNN-Comparable under optimal hyperparameter configurations.}}
    \label{fig:training_curves}
\end{figure}

\subsubsection{Sample Efficiency}

To evaluate the sample efficiency of each model, we trained RVNN, CVNN, and DualNN under varying levels of data availability, specifically with 500, 1000, 2000, 4000, 6000, 8000, and 10000 samples per class. 
The results are presented in Fig.~\ref{fig:sample_efficiency}.
CVNN exhibited poor performance in the low-data regimes, achieving substantially lower accuracy at 500, 1000, and 2000 samples per class, suggesting that its complex-valued parameterization requires a relatively large number of samples to learn effectively. 
RVNN showed a progressive improvement as the number of training samples increased, though its performance remained notably sensitive to data availability. 
In contrast, DualNN demonstrated a markedly more stable accuracy across all data regimes. 
This confirms that the dual design effectively utilizes the intrinsic statistical symmetry of I/Q channels, allowing the model to learn meaningful features from significantly fewer observations.

\begin{figure}
    \centering
    \includegraphics[width=0.9\linewidth]{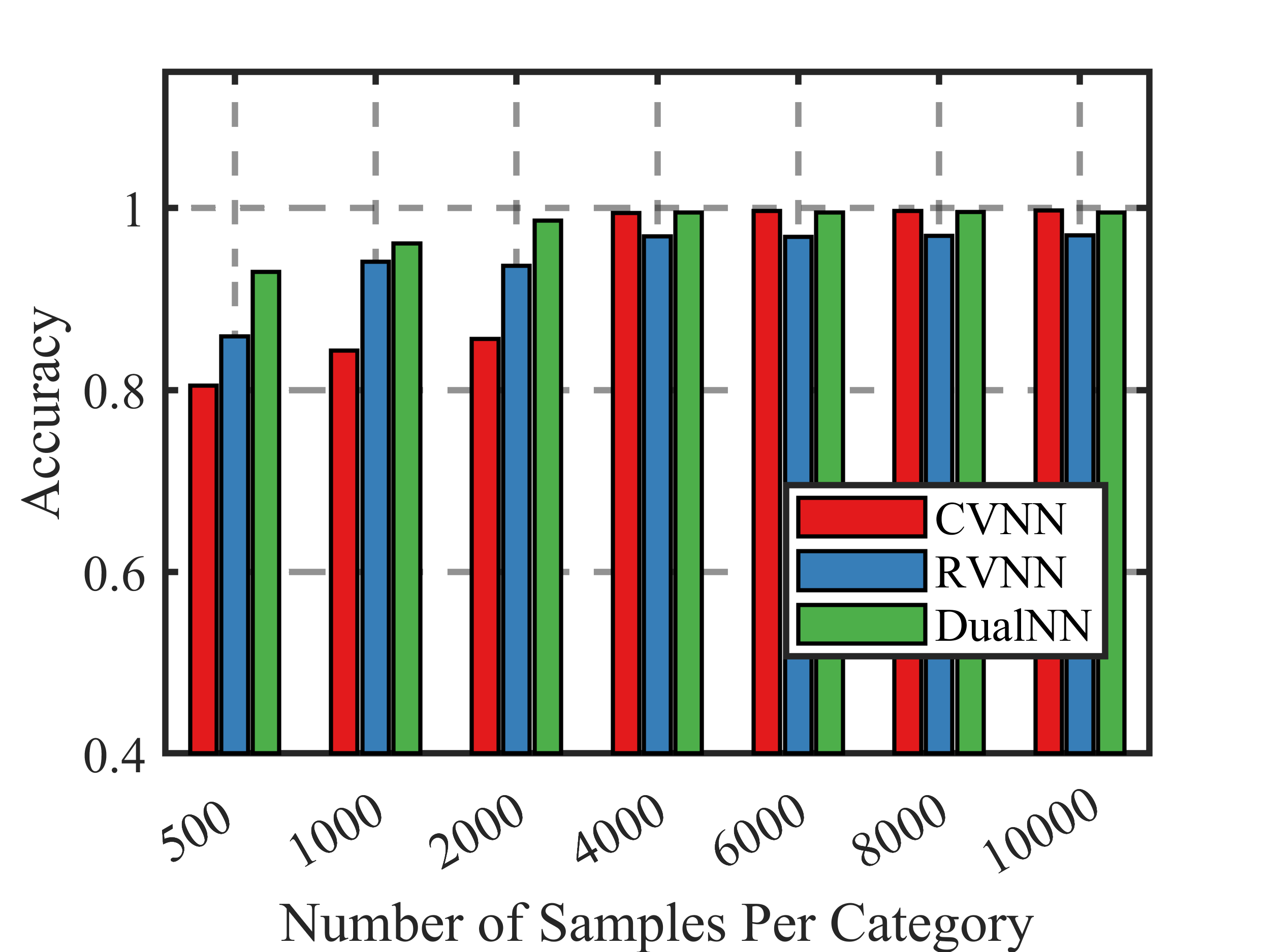}
    \caption{\textcolor{black}{Classification accuracy of RVNN, CVNN, and DualNN as a function of dataset size, evaluated at 7 data scales, i.e. 500, 1000, 2000, 4000, 6000, 8000, and 10000 samples per class.}}

    \label{fig:sample_efficiency}
\end{figure}

\subsubsection{Effect of I/Q Imbalance}

In practical communication systems, hardware imperfections inevitably introduce I/Q imbalance into the received signal \cite{li2013frequency, wong2019specific}. Under such conditions, the signal $x_c(n)$ in Eq.~(3) can be expressed as:
\begin{equation}
\begin{array}{c}
{{\tilde x}_c}\left( n \right) = (1 + \alpha ){x_{real}}\left( n \right)\cos \left( {2\pi {f_c}n + \phi } \right)\\
 - j{x_{imag}}\left( n \right)\sin \left( {2\pi {f_c}n} \right)
\end{array}
\end{equation}
where $\alpha$ denotes the amplitude mismatch factor and $\phi$ denotes the phase mismatch angle. For each sample, $\alpha$ and $\phi$ are independently and randomly drawn from their respective ranges, and five severity levels are defined to systematically assess robustness.

The test accuracy of DualNN, RVNN, and CVNN under five I/Q imbalance conditions is reported in Table~\ref{tab:iq_imbalance}. 
All three architectures exhibit monotonic degradation as imperfection severity increases, yet the magnitudes differ markedly. RVNN suffers the largest drop of 2.12\%, attributable to its discarding of the imaginary component, which reduces both baseline accuracy and resilience to distributional shift. 
The performance degradation of DualNN and CVNN reflects their joint exploitation of both I and Q channels.
The marginal advantage of CVNN under severe conditions is consistent with its complex-valued weights, which inherently couple I and Q and can partially compensate for channel asymmetry.
DualNN's shared real-valued parameterization implicitly assumes I/Q symmetry, incurring a slight additional penalty when that assumption is violated. Nonetheless, the gap between DualNN and CVNN remains within 0.10\% across all conditions, confirming that DualNN's parameter-efficient design entails no meaningful robustness cost. 
The absence of catastrophic collapse across all models reflects the tractability of the five-class task and the diversity afforded by 4000 training samples per class.

\color{black}

\begin{table}[!t]
\renewcommand{\arraystretch}{1.3}
\caption{{Effect of I/Q Imbalance.}}
\label{tab:iq_imbalance}
\centering
\begin{tabular}{lccccc}
\hline
\multirow{2}{*}{{Imperfection}}  & \multirow{2}{*}{{$\alpha$ Range}}  & \multirow{2}{*}{{$\theta$ Range}}  & \multicolumn{3}{c}{{Test Acccuracy}} \\
\cline{4-6}
 ~ & ~ & ~ & DualNN & RVNN & CVNN \\
\hline
Ideal & 0 & $0^\circ$ & 99.50\% & 96.87\% & 99.35\% \\
Mild & $[0, 0.03]$ & $[0^\circ, 2^\circ]$ & 99.11\% & 96.33\% & 99.21\%\\
Moderate & $[0, 0.06]$ & $[0^\circ, 4^\circ]$ & 98.96\% & 96.12\% & 99.06\%\\
Severe & $[0, 0.10]$ & $[0^\circ, 6^\circ]$ & 98.86\% & 95.93\%& 99.01\%\\
Extreme & $[0, 0.15]$ & $[0^\circ, 8^\circ]$ & 98.71\% & 94.75\%& 98.61\%\\
\hline
\end{tabular}
\end{table}

\subsection{Comparison with RVNN and CVNN in Different Tasks}
In this experiment, we examined the impact of dual-channel architecture on processing complex-valued signals. We compared this approach with both RVNNs and CVNNs across SNR ranging from -4 dB to 12 dB in three tasks: AMR, SSR, and SSP. 
To ensure fair comparison, we maintained the same network architecture across all implementations, varying only the channel configuration. Models processing only the I-channel with real-valued parameters are designated as ``Realformer'', while those processing the full complex-valued signal with complex-valued parameters are labeled ``Complexformer''.

The comparative performance of Dualformer versus RVNN and CVNN on the AMR, SSR, and SSP datasets (described in Section~V.A) across varying SNR levels is shown in Fig.~\ref{fig:exp1-amr}, Fig.~\ref{fig:exp1-ssr}, and Fig.~\ref{fig:exp1-ssp}, respectively. As shown in the figures, Dualformer exhibits superior performance throughout the entire SNR range in terms of both accuracy and F1-score metrics, particularly in the AMR task.
\color{black}
Notably, at an extremely low SNR of $-4$ dB, the performance degradation of Dualformer is substantially smaller than that of traditional CVNNs. In such severe noise conditions, the local amplitude and phase structures are heavily corrupted, causing conventional local CNN filters to fail. The robustness of Dualformer stems primarily from the self-attention mechanism, which can capture and aggregate long-range phase dependencies across the entire sequence, effectively reconstructing the underlying signal characteristics despite the heavy localized noise.
\color{black}

\begin{figure}[!t]
\centering
\subfloat[]{\includegraphics[width=0.9\linewidth]{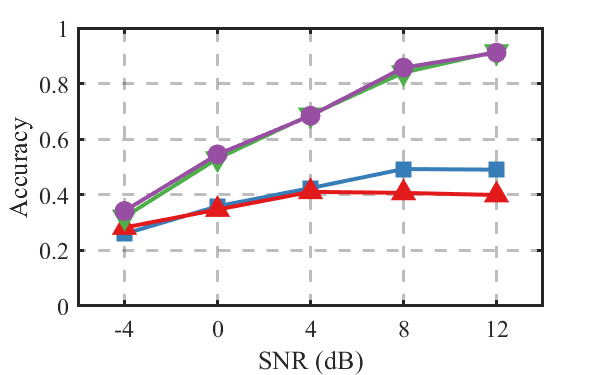}}
\\
\subfloat[]{\includegraphics[width=0.9\linewidth]{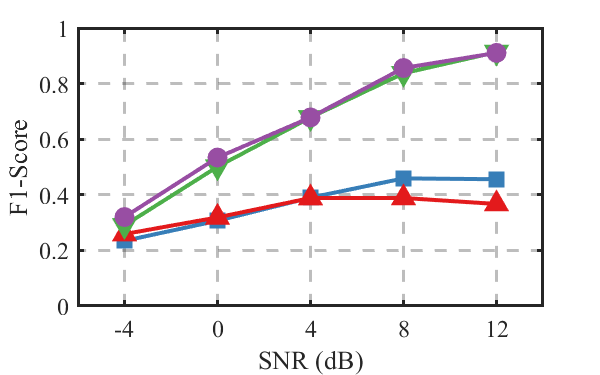}}
\\
\includegraphics[width=0.9\linewidth]{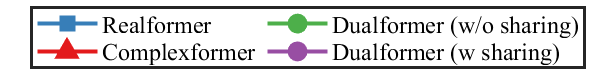}
\caption{Comparison of real-valued, complex-valued, and Dualformer architectures in AMR.
(a) Accuracy. 
(b) F1-score. 
 }
\label{fig:exp1-amr}
\end{figure}

\begin{figure}[!t]
\centering
\subfloat[]{\includegraphics[width=0.9\linewidth]{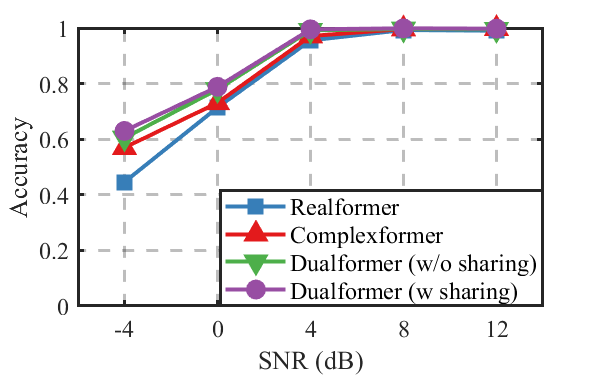}}
\\
\subfloat[]{\includegraphics[width=0.9\linewidth]{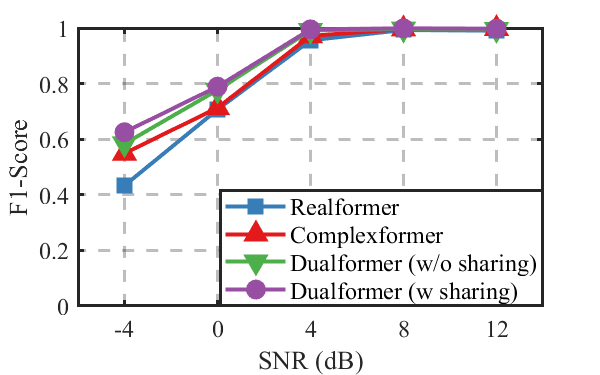}}
\caption{Comparison of real-valued, complex-valued, and  Dualformer architectures in SSR.
(a) Accuracy. 
(b) F1-score. 
 }
\label{fig:exp1-ssr}
\end{figure}

\begin{figure}[!t]
\centering
\subfloat[]{\includegraphics[width=0.9\linewidth]{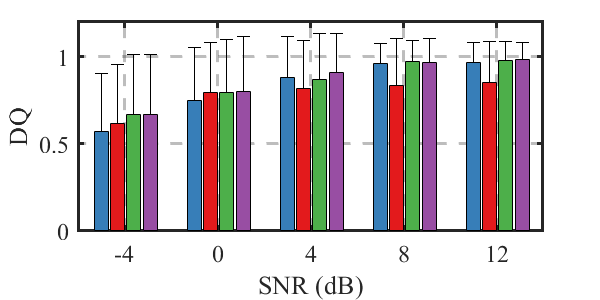}}
\\
\subfloat[]{\includegraphics[width=0.9\linewidth]{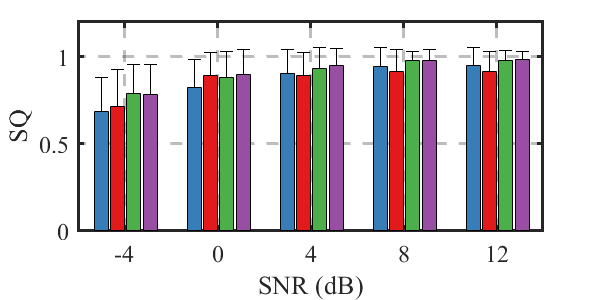}}
\\
\subfloat[]{\includegraphics[width=0.9\linewidth]{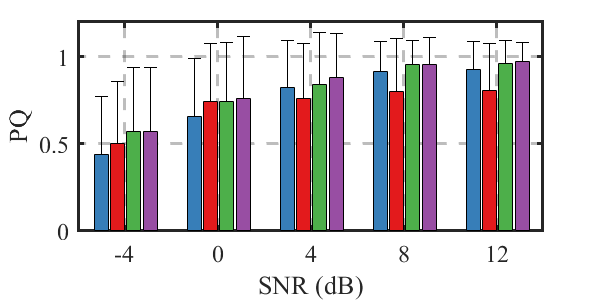}}
\\
\includegraphics[width=0.9\linewidth]{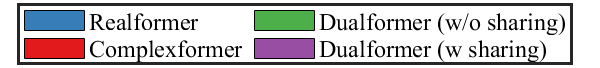}
\caption{Comparison of real-valued, complex-valued, and 
Dualformer architectures in SSP.
(a) RQ. 
(b) SQ.
(c) PQ.
 }
\label{fig:exp1-ssp}
\end{figure}

\subsection{Comparison with Benchmarks in Coarse-Grained Analysis}
\begin{table*}[!t]
\caption{Comparion in AMR.\label{tab:compare_amr}}
\centering
\resizebox{\textwidth}{!}{
\begin{tabular}{c|c|c|c|c|c|c|c|c|c|c}
\hline
\multirow{3}{*}{Method} & \multicolumn{10}{c}{SNR(dB)}                                                                                                \\ \cline{2-11} 
                        & \multicolumn{2}{c|}{-4} & \multicolumn{2}{c|}{0} & \multicolumn{2}{c|}{4} & \multicolumn{2}{c|}{8} & \multicolumn{2}{c}{12} \\ \cline{2-11} 
 &
  \multicolumn{1}{c|}{Accuracy} &
  \multicolumn{1}{c|}{F1-Score} &
  \multicolumn{1}{c|}{Accuracy} &
  \multicolumn{1}{c|}{F1-Score} &
  \multicolumn{1}{c|}{Accuracy} &
  \multicolumn{1}{c|}{F1-Score} &
  \multicolumn{1}{c|}{Accuracy} &
  \multicolumn{1}{c|}{F1-Score} &
  \multicolumn{1}{c|}{Accuracy} &
  F1-Score \\ \hline
CNN1 
&0.0567&	0.0349
&0.2361&	0.1992
&0.2726&	0.2346
&0.3919&	0.3461
&0.3420 &	0.2829
\\ \hline
CNN2
&0.3011&	0.2794
&0.5132&	0.4904
&0.6142&	0.5881
&0.7166&	0.6907
&0.7668&	0.7402
\\ \hline
LSTM
&0.0832&	0.0230
&0.3040&	0.2758
&0.6245&	0.5831
&0.6390&	0.6234
&0.8485&	0.8403
\\ \hline
ResNet
&0.0419&	0.0033
&0.4137&	0.3538
&0.5682&	0.5304
&0.6695&	0.6557
&0.7491&	0.7352
\\ \hline
Autoformer
&0.1982&	0.1662
&0.3269&	0.2749
&0.4719&	0.4426
&0.5121&	0.5013
&0.4086&	0.3617
\\ \hline
Informer
&\textbf{0.3577}	&\textbf{0.3414}
&0.5424	&0.5254
&0.6767	&0.6613
&0.8049	&0.8012
&0.8876	&0.8858
\\ \hline
Transformer
&0.3105&	0.2989
&0.5279&	0.5166
&0.6244&	0.6054
&0.7270&	0.7225
&0.8172&	0.8131
\\ \hline
Dualformer
&0.3419	&0.3200
&\textbf{0.5452}	&\textbf{0.5340}
&\textbf{0.6843}	&\textbf{0.6783}
&\textbf{0.8574}	&\textbf{0.8563}
&\textbf{0.9116}	&\textbf{0.9105}
\\ \hline
\end{tabular}
}
\end{table*}

\begin{table*}[!t]
\caption{Comparion in SSR.\label{tab:compare_amr}}
\centering
\resizebox{\textwidth}{!}{
\begin{tabular}{c|c|c|c|c|c|c|c|c|c|c}
\hline
\multirow{3}{*}{Method} & \multicolumn{10}{c}{SNR(dB)}           \\ \cline{2-11} 
& \multicolumn{2}{c|}{-4} & \multicolumn{2}{c|}{0} & \multicolumn{2}{c|}{4} & \multicolumn{2}{c|}{8} & \multicolumn{2}{c}{12} \\ \cline{2-11} 
 &
  \multicolumn{1}{c|}{Accuracy} &
  \multicolumn{1}{c|}{F1-Score} &
  \multicolumn{1}{c|}{Accuracy} &
  \multicolumn{1}{c|}{F1-Score} &
  \multicolumn{1}{c|}{Accuracy} &
  \multicolumn{1}{c|}{F1-Score} &
  \multicolumn{1}{c|}{Accuracy} &
  \multicolumn{1}{c|}{F1-Score} &
  \multicolumn{1}{c|}{Accuracy} &
  F1-Score \\ \hline
CNN1 
&0.2813	&0.2343
&0.4519	&0.3985
&0.5084	&0.4509
&0.5361	&0.4766
&0.5697	&0.5121
\\ \hline
CNN2
&0.3029	&0.2446
&0.4940	&0.4162
&0.5397	&0.4560
&0.5493	&0.4745
&0.7224	&0.6455

\\ \hline
LSTM
&0.0541	&0.0206
&0.0733	&0.0340
&0.0961	&0.0684
&0.1190	&0.0749
&0.1370	&0.0861

\\ \hline
ResNet
&0.1971	&0.1141
&0.3582	&0.2742
&0.3906	&0.2980
&0.3966	&0.2952
&0.4243	&0.3242

\\ \hline
Autoformer
&0.1791	&0.1390
&0.4171	&0.3549
&0.4339	&0.3754
&0.2837	&0.2186
&0.4339	&0.3987

\\ \hline
Informer
&0.3341	&0.2975
&0.6623	&0.6459
&0.9664	&0.9665
&0.9411	&0.9410
&0.9928	&0.9928

\\ \hline
Transformer
&0.3149	&0.2911
&0.6344	&0.6162
&0.9340	&0.9340
&0.9906	&0.9904
&0.9965	&0.9964
\\ \hline
Dualformer
&\textbf{0.6298}	&\textbf{0.6252}
&\textbf{0.7885}	&\textbf{0.7892}
&\textbf{0.9952}	&\textbf{0.9952}
&\textbf{0.9988}	&\textbf{0.9988}
&\textbf{0.9976}	&\textbf{0.9976}
\\ \hline
\end{tabular}
}
\end{table*}

To rigorously evaluate the performance of our algorithm in both AMR and SSR tasks, we conduct comparative experiments with four representative DL architectures and three state-of-the-art Transformer variants. The baseline methods include CNN1 \cite{o2016convolutional}, CNN2 \cite{tekbiyik2020robust}, LSTM \cite{rajendran2018deep}, ResNet \cite{liu2017deep}, Transformer \cite{rashvand2024enhancing}, Informer \cite{han2024lightweight}, and Autoformer \cite{wu2021autoformer}.

Tables IV and V summarize the comparative performance metrics (accuracy and F1-Score) of all evaluated methods for AMR and SSR tasks, respectively.
As illustrated in the table, Transformer architectures consistently surpass classical DL models in recognition accuracy. Particularly, the proposed Dualformer establishes new state-of-the-art performance among all other transformer variants across the entire SNR range, exhibiting monotonically improving performance gains with increasing SNR levels.
In summary, Transformers outperform classical DL models in feature extraction. However, optimal performance requires architectures specifically designed for signal data.

\subsection{Comparison in Fine-grained Analysis}
Since classical DL models cannot handle SSP tasks, we compare Dualformer only with Transformer-based methods. As shown in Fig.~\ref{fig:exp2-ssp}, Dualformer consistently outperforms these approaches in RQ, SQ, and PQ metrics across all SNRs (-4dB to 12dB), with the performance gap widening as SNR increases.

\begin{figure}[!t]
\centering
\subfloat[]{\includegraphics[width=0.9\linewidth]{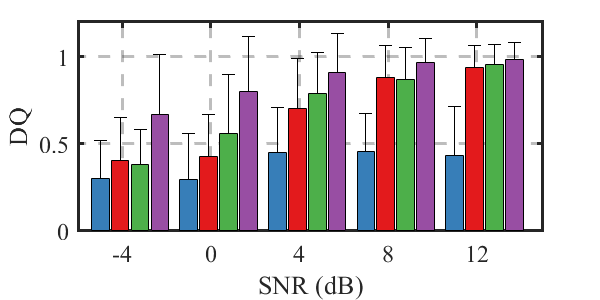}}
\\
\subfloat[]{\includegraphics[width=0.9\linewidth]{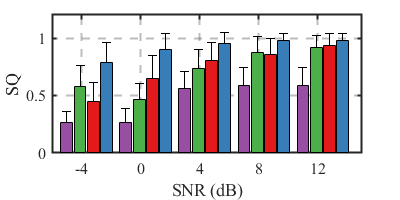}}
\\
\subfloat[]{\includegraphics[width=0.9\linewidth]{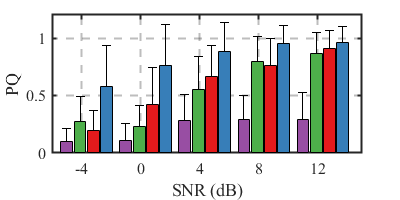}}
\\
\includegraphics[width=\linewidth]{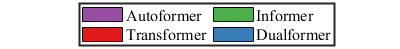}
\caption{Comparison of Transformer-based architectures in SSP.
(a) RQ. 
(b) SQ.
(c) PQ.
 }
\label{fig:exp2-ssp}
\end{figure}

\section{Conclusion}
Effective feature extraction is critical for blind signal analysis tasks such as AMR, SSR, and SSP. To effectively leverage the real and imaginary part of the complex-valued signals, we propose a novel dual-channel neural network framework, called DualNN, which designs a parameter-shared architecture across IQ channels. Theoretical analysis demonstrates that DualNN achieves reduced generalization error while maintaining expressive capacity. 
Furthermore, we introduce Dualformer, a Transformer-based architecture specifically designed to validate the advantages of DualNN. Extensive comparisons with three Transformer-based and four conventional DL-based methods validate the advantages of Dualformer, which achieves significant performance gains in AMR, SSR, and SSP.

DualNN provides a new paradigm for processing the complex-valued signals in the context of blind signal analysis. 
\color{black}
The modular design of DualNN presents exciting opportunities for integration into more advanced signal processing frameworks, particularly in tasks like blind source separation (BSS) and low-SNR spectrum sensing. For example, DualNN can be naturally integrated as a deep feature-extraction front-end within independent component analysis (ICA) architectures. By explicitly exploiting the shared temporal structures of the I and Q paths through parameter sharing, DualNN reduces the dimension of latent space, thereby projecting the mixtures into a superior, noise-resilient latent space for more accurate signal de-mixing. 
\color{black}
As such, this work lays a foundational step toward the broader application in blind signal processing, and opens up promising avenues for future research across a wide range of unsupervised and weakly supervised signal analysis scenarios.

\bibliographystyle{IEEEtran}
\bibliography{reference.bib}

\begin{thebibliography}{10}
\providecommand{\url}[1]{#1}
\csname url@samestyle\endcsname
\providecommand{\newblock}{\relax}
\providecommand{\bibinfo}[2]{#2}
\providecommand{\BIBentrySTDinterwordspacing}{\spaceskip=0pt\relax}
\providecommand{\BIBentryALTinterwordstretchfactor}{4}
\providecommand{\BIBentryALTinterwordspacing}{\spaceskip=\fontdimen2\font plus
\BIBentryALTinterwordstretchfactor\fontdimen3\font minus \fontdimen4\font\relax}
\providecommand{\BIBforeignlanguage}[2]{{%
\expandafter\ifx\csname l@#1\endcsname\relax
\typeout{** WARNING: IEEEtran.bst: No hyphenation pattern has been}%
\typeout{** loaded for the language `#1'. Using the pattern for}%
\typeout{** the default language instead.}%
\else
\language=\csname l@#1\endcsname
\fi
#2}}
\providecommand{\BIBdecl}{\relax}
\BIBdecl

\bibitem{liu2022integrated}
F.~Liu, Y.~Cui, C.~Masouros, J.~Xu, T.~X. Han, Y.~C. Eldar, and S.~Buzzi, ``Integrated sensing and communications: Toward dual-functional wireless networks for 6g and beyond,'' \emph{IEEE journal on selected areas in communications}, vol.~40, no.~6, pp. 1728--1767, 2022.

\bibitem{de2008blind}
P.~De and Y.-C. Liang, ``Blind spectrum sensing algorithms for cognitive radio networks,'' \emph{IEEE transactions on vehicular technology}, vol.~57, no.~5, pp. 2834--2842, 2008.

\bibitem{awin2018blind}
F.~Awin, E.~Abdel-Raheem, and K.~Tepe, ``Blind spectrum sensing approaches for interweaved cognitive radio system: A tutorial and short course,'' \emph{IEEE Communications Surveys \& Tutorials}, vol.~21, no.~1, pp. 238--259, 2018.

\bibitem{chowdhury20206g}
M.~Z. Chowdhury, M.~Shahjalal, S.~Ahmed, and Y.~M. Jang, ``6g wireless communication systems: Applications, requirements, technologies, challenges, and research directions,'' \emph{IEEE Open Journal of the Communications Society}, vol.~1, pp. 957--975, 2020.

\bibitem{11431942}
J.~Lei, Y.~Li, Z.~Wang, Q.~Lin, Y.-F. Liu, and Y.-C. Wu, ``A unified distributed algorithm for hybrid near-far field activity detection in cell-free massive mimo,'' \emph{arXiv preprint arXiv:2509.15162}, 2025.

\bibitem{becker2023predicting}
\BIBentryALTinterwordspacing
S.~Becker, M.~Klein, A.~Neitz, G.~Parascandolo, and N.~Kilbertus, ``Predicting ordinary differential equations with transformers,'' 2023, published at the 40th International Conference on Machine Learning (ICML 2023). [Online]. Available: \url{https://arxiv.org/abs/2307.12617}
\BIBentrySTDinterwordspacing

\bibitem{Yang2023FLUIDGPTL}
S.~D. Yang, Z.~A. Ali, and B.~M. Wong, ``Fluid-gpt (fast learning to understand and investigate dynamics with a generative pre-trained transformer): Efficient predictions of particle trajectories and erosion,'' \emph{Industrial {\&} Engineering Chemistry Research}, vol.~62, no.~37, pp. 15\,278--15\,289, 2023.

\bibitem{nie2023patchtst}
\BIBentryALTinterwordspacing
Y.~Nie, N.~H. Nguyen, P.~Sinthong, and J.~Kalagnanam, ``A time series is worth 64 words: Long-term forecasting with transformers,'' in \emph{The Eleventh International Conference on Learning Representations (ICLR)}, 2023, accepted for publication at ICLR 2023. [Online]. Available: \url{https://openreview.net/forum?id=Jbdc0vTOcol}
\BIBentrySTDinterwordspacing

\bibitem{liu2023itransformer}
\BIBentryALTinterwordspacing
Y.~Liu, T.~Hu, H.~Zhang, H.~Wu, S.~Wang, L.~Ma, and M.~Long, ``itransformer: Inverted transformers are effective for time series forecasting,'' 2023. [Online]. Available: \url{https://arxiv.org/abs/2310.06625}
\BIBentrySTDinterwordspacing

\bibitem{o2016convolutional}
T.~J. O’Shea, J.~Corgan, and T.~C. Clancy, ``Convolutional radio modulation recognition networks,'' in \emph{Engineering Applications of Neural Networks: 17th International Conference, EANN 2016, Aberdeen, UK, September 2-5, 2016, Proceedings 17}.\hskip 1em plus 0.5em minus 0.4em\relax Springer, 2016, pp. 213--226.

\bibitem{zhang2024class}
Z.~Zhang, M.~Zhu, J.~Liu, Y.~Li, and S.~Wang, ``Class information guided reconstruction for automatic modulation open-set recognition,'' \emph{IEEE Transactions on Cognitive Communications and Networking}, 2024.

\bibitem{lin2020improved}
Y.~Lin, Y.~Tu, and Z.~Dou, ``An improved neural network pruning technology for automatic modulation classification in edge devices,'' \emph{IEEE Transactions on Vehicular Technology}, vol.~69, no.~5, pp. 5703--5706, 2020.

\bibitem{rajendran2018deep}
S.~Rajendran, W.~Meert, D.~Giustiniano, V.~Lenders, and S.~Pollin, ``Deep learning models for wireless signal classification with distributed low-cost spectrum sensors,'' \emph{IEEE Transactions on Cognitive Communications and Networking}, vol.~4, no.~3, pp. 433--445, 2018.

\bibitem{ke2021real}
Z.~Ke and H.~Vikalo, ``Real-time radio technology and modulation classification via an lstm auto-encoder,'' \emph{IEEE Transactions on Wireless Communications}, vol.~21, no.~1, pp. 370--382, 2021.

\bibitem{zhang2020automatic}
Z.~Zhang, H.~Luo, C.~Wang, C.~Gan, and Y.~Xiang, ``Automatic modulation classification using cnn-lstm based dual-stream structure,'' \emph{IEEE Transactions on Vehicular Technology}, vol.~69, no.~11, pp. 13\,521--13\,531, 2020.

\bibitem{wang2021multidimensional}
N.~Wang, Y.~Liu, L.~Ma, Y.~Yang, and H.~Wang, ``Multidimensional cnn-lstm network for automatic modulation classification,'' \emph{Electronics}, vol.~10, no.~14, p. 1649, 2021.

\bibitem{li2022transformer}
L.~Li, C.~Qin, G.~Li, S.~Hu, Y.~Xie, and Z.~Lei, ``Transformer-based radio modulation mode recognition,'' in \emph{Journal of Physics: Conference Series}, vol. 2384, no.~1.\hskip 1em plus 0.5em minus 0.4em\relax IOP Publishing, 2022, p. 012017.

\bibitem{zheng2022tmrn}
Y.~Zheng, Y.~Ma, and C.~Tian, ``Tmrn-glu: A transformer-based automatic classification recognition network improved by gate linear unit,'' \emph{Electronics}, vol.~11, no.~10, p. 1554, 2022.

\bibitem{chen2022abandon}
Y.~Chen, B.~Dong, C.~Liu, W.~Xiong, and S.~Li, ``Abandon locality: Frame-wise embedding aided transformer for automatic modulation recognition,'' \emph{IEEE Communications Letters}, vol.~27, no.~1, pp. 327--331, 2022.

\bibitem{tu2020complex}
Y.~Tu, Y.~Lin, C.~Hou, and S.~Mao, ``Complex-valued networks for automatic modulation classification,'' \emph{IEEE Transactions on Vehicular Technology}, vol.~69, no.~9, pp. 10\,085--10\,089, 2020.

\bibitem{xiao2023complex}
C.~Xiao, S.~Yang, and Z.~Feng, ``Complex-valued depthwise separable convolutional neural network for automatic modulation classification,'' \emph{IEEE Transactions on Instrumentation and Measurement}, vol.~72, pp. 1--10, 2023.

\bibitem{luo2023complex}
Q.~Luo, M.-M. Zhao, Z.~Chen, Z.~Su, and M.-J. Zhao, ``Complex-valued convolution and frequency global filter for automatic modulation recognition,'' \emph{IEEE Communications Letters}, vol.~27, no.~7, pp. 1779--1783, 2023.

\bibitem{ren2022complex}
Y.~Ren, W.~Jiang, and Y.~Liu, ``Complex-valued parallel convolutional recurrent neural networks for automatic modulation classification,'' in \emph{2022 IEEE 25th International Conference on Computer Supported Cooperative Work in Design (CSCWD)}.\hskip 1em plus 0.5em minus 0.4em\relax IEEE, 2022, pp. 804--809.

\bibitem{yang2022automatic}
X.~Yang, R.~Zhang, H.~Xie, H.~Sun, and H.~Li, ``Automatic modulation mode recognition of communication signals based on complex-valued neural network,'' in \emph{2022 International Conference on Computing, Communication, Perception and Quantum Technology (CCPQT)}.\hskip 1em plus 0.5em minus 0.4em\relax IEEE, 2022, pp. 32--37.

\bibitem{xu2023novel}
Z.~Xu, S.~Hou, S.~Fang, H.~Hu, and Z.~Ma, ``A novel complex-valued hybrid neural network for automatic modulation classification,'' \emph{Electronics}, vol.~12, no.~20, p. 4380, 2023.

\bibitem{lei2024understanding}
J.~Lei, Y.~Li, L.-Y. Yung, Y.~Leng, Q.~Lin, and Y.-C. Wu, ``Understanding complex-valued transformer for modulation recognition,'' \emph{IEEE Wireless Communications Letters}, 2024.

\bibitem{leng2025unveiling}
Y.~Leng, Q.~Lin, L.-Y. Yung, J.~Lei, Y.~Li, and Y.-C. Wu, ``Unveiling the power of complex-valued transformers in wireless communications,'' \emph{IEEE Transactions on Communications}, vol.~74, pp. 612--627, 2026.

\bibitem{LiDWYH24}
\BIBentryALTinterwordspacing
W.~Li, W.~Deng, K.~Wang, L.~You, and Z.~Huang, ``A complex-valued transformer for automatic modulation recognition,'' \emph{{IEEE} Internet Things J.}, vol.~11, no.~12, pp. 22\,197--22\,207, 2024. [Online]. Available: \url{https://doi.org/10.1109/JIOT.2024.3379429}
\BIBentrySTDinterwordspacing

\bibitem{wang2025survey}
X.~Wang, Y.~Zhao, and Z.~Huang, ``A survey of deep transfer learning in automatic modulation classification,'' \emph{IEEE Transactions on Cognitive Communications and Networking}, 2025.

\bibitem{jakubovitz2019generalization}
D.~Jakubovitz, R.~Giryes, and M.~R. Rodrigues, ``Generalization error in deep learning,'' in \emph{Compressed sensing and its applications: third international MATHEOn conference 2017}.\hskip 1em plus 0.5em minus 0.4em\relax Springer, 2019, pp. 153--193.

\bibitem{stinchcomb1989multilayered}
M.~Stinchcomb, ``Multilayered feedforward networks are universal approximators,'' \emph{Neural Networks}, vol.~2, pp. 356--359, 1989.

\bibitem{Jiang2021}
H.~Jiang, Z.~Li, and Q.~Li, ``Approximation theory of convolutional architectures for time series modelling,'' in \emph{Proceedings of the 38th International Conference on Machine Learning}.\hskip 1em plus 0.5em minus 0.4em\relax PMLR, July 2021, pp. 4961--4970.

\bibitem{Li2021}
Z.~Li, H.~Jiang, and Q.~Li, ``On the approximation properties of recurrent encoder-decoder architectures,'' in \emph{International Conference on Learning Representations}, September 2021.

\bibitem{Li2022}
Z.~Li, J.~Han, W.~E, and Q.~Li, ``Approximation and optimization theory for linear continuous-time recurrent neural networks,'' \emph{Journal of Machine Learning Research}, vol.~23, no.~42, pp. 1--85, 2022.

\bibitem{jiang2024approximation}
H.~Jiang and Q.~Li, ``Approximation rate of the transformer architecture for sequence modeling,'' \emph{Advances in Neural Information Processing Systems}, vol.~37, pp. 68\,926--68\,955, 2024.

\bibitem{girosi1995approximation}
F.~Girosi, ``Approximation error bounds that use vc-bounds,'' in \emph{Proc. International Conference on Artificial Neural Networks, F. Fogelman-Soulie and P. Gallinari, editors}, vol.~1, 1995, pp. 295--302.

\bibitem{barron1994approximation}
A.~R. Barron, ``Approximation and estimation bounds for artificial neural networks,'' \emph{Machine learning}, vol.~14, pp. 115--133, 1994.

\bibitem{Bartlett2017}
P.~L. {}Bartlett, D.~J. Foster, and M.~J. Telgarsky, ``Spectrally-normalized margin bounds for neural networks,'' \emph{Advances in neural information processing systems}, vol.~30, 2017.

\bibitem{chen2023spectral}
H.~Chen, F.~He, S.~Lei, and D.~Tao, ``Spectral complexity-scaled generalisation bound of complex-valued neural networks,'' \emph{Artificial Intelligence}, vol. 322, p. 103951, 2023.

\bibitem{o2018over}
T.~J. O’Shea, T.~Roy, and T.~C. Clancy, ``Over-the-air deep learning based radio signal classification,'' \emph{IEEE Journal of Selected Topics in Signal Processing}, vol.~12, no.~1, pp. 168--179, 2018.

\bibitem{hossin2015review}
M.~Hossin and M.~N. Sulaiman, ``A review on evaluation metrics for data classification evaluations,'' \emph{International journal of data mining \& knowledge management process}, vol.~5, no.~2, p.~1, 2015.

\bibitem{kirillov2019panoptic}
A.~Kirillov, K.~He, R.~Girshick, C.~Rother, and P.~Doll{\'a}r, ``Panoptic segmentation,'' in \emph{Proceedings of the IEEE/CVF conference on computer vision and pattern recognition}, 2019, pp. 9404--9413.

\bibitem{li2013frequency}
Y.~Li, ``Frequency independent iq imbalance estimation and compensation,'' in \emph{In-Phase and Quadrature Imbalance: Modeling, Estimation, and Compensation}.\hskip 1em plus 0.5em minus 0.4em\relax Springer, 2013, pp. 29--47.

\bibitem{wong2019specific}
L.~J. Wong, W.~C. Headley, and A.~J. Michaels, ``Specific emitter identification using convolutional neural network-based iq imbalance estimators,'' \emph{IEEE Access}, vol.~7, pp. 33\,544--33\,555, 2019.

\bibitem{tekbiyik2020robust}
K.~Tekb{\i}y{\i}k, A.~R. Ekti, A.~G{\"o}r{\c{c}}in, G.~K. Kurt, and C.~Ke{\c{c}}eci, ``Robust and fast automatic modulation classification with cnn under multipath fading channels,'' in \emph{2020 IEEE 91st Vehicular Technology Conference (VTC2020-Spring)}.\hskip 1em plus 0.5em minus 0.4em\relax IEEE, 2020, pp. 1--6.

\bibitem{liu2017deep}
X.~Liu, D.~Yang, and A.~El~Gamal, ``Deep neural network architectures for modulation classification,'' in \emph{2017 51st Asilomar Conference on Signals, Systems, and Computers}.\hskip 1em plus 0.5em minus 0.4em\relax IEEE, 2017, pp. 915--919.

\bibitem{rashvand2024enhancing}
N.~Rashvand, K.~Witham, G.~Maldonado, V.~Katariya, N.~Marer~Prabhu, G.~Schirner, and H.~Tabkhi, ``Enhancing automatic modulation recognition for iot applications using transformers,'' \emph{IoT}, vol.~5, no.~2, pp. 212--226, 2024.

\bibitem{han2024lightweight}
J.~Han, Z.~Yu, and J.~Yang, ``A lightweight deep learning architecture for automatic modulation classification of wireless internet of things,'' \emph{IET Communications}, vol.~18, no.~18, pp. 1220--1230, 2024.

\bibitem{wu2021autoformer}
H.~Wu, J.~Xu, J.~Wang, and M.~Long, ``Autoformer: Decomposition transformers with auto-correlation for long-term series forecasting,'' \emph{Advances in neural information processing systems}, vol.~34, pp. 22\,419--22\,430, 2021.

\end{thebibliography}

\begin{IEEEbiography}[{\includegraphics[width=1in,height=1.25in,clip,keepaspectratio]{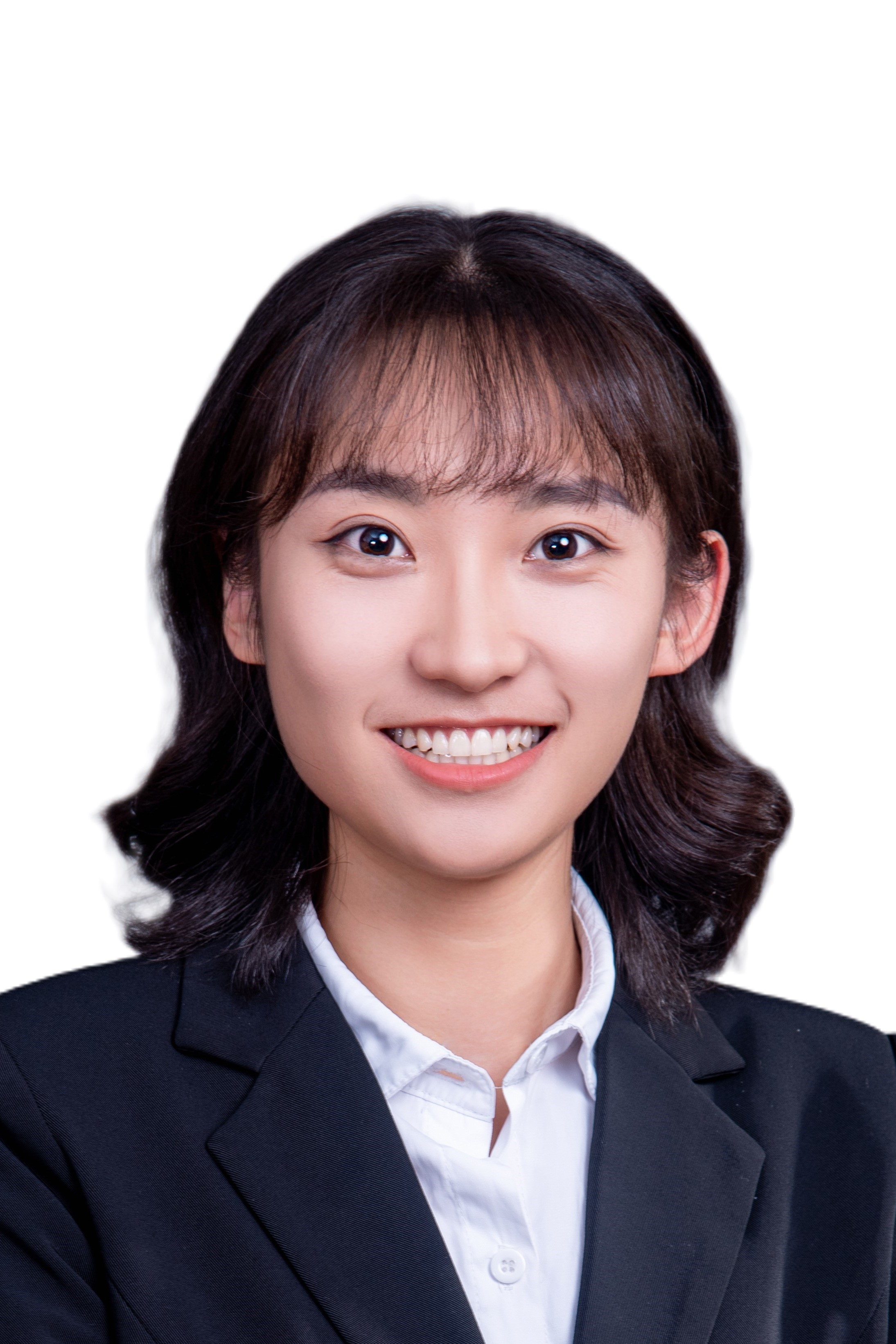}}]{Yurui Zhao}
received her B.Eng. and M.Phil. degrees in Electronic Science and Engineering from the National University of Defense Technology. She is currently pursuing a Ph.D. degree with the State Key Laboratory of Complex Electromagnetic Environment Effects on Electronics and Information Systems, National University of Defense Technology. 
Her current research interests include representation learning, non-cooperative signal processing, and deep learning. 
\end{IEEEbiography}

\begin{IEEEbiography}[{\includegraphics[width=1in,height=1.25in,clip,keepaspectratio]{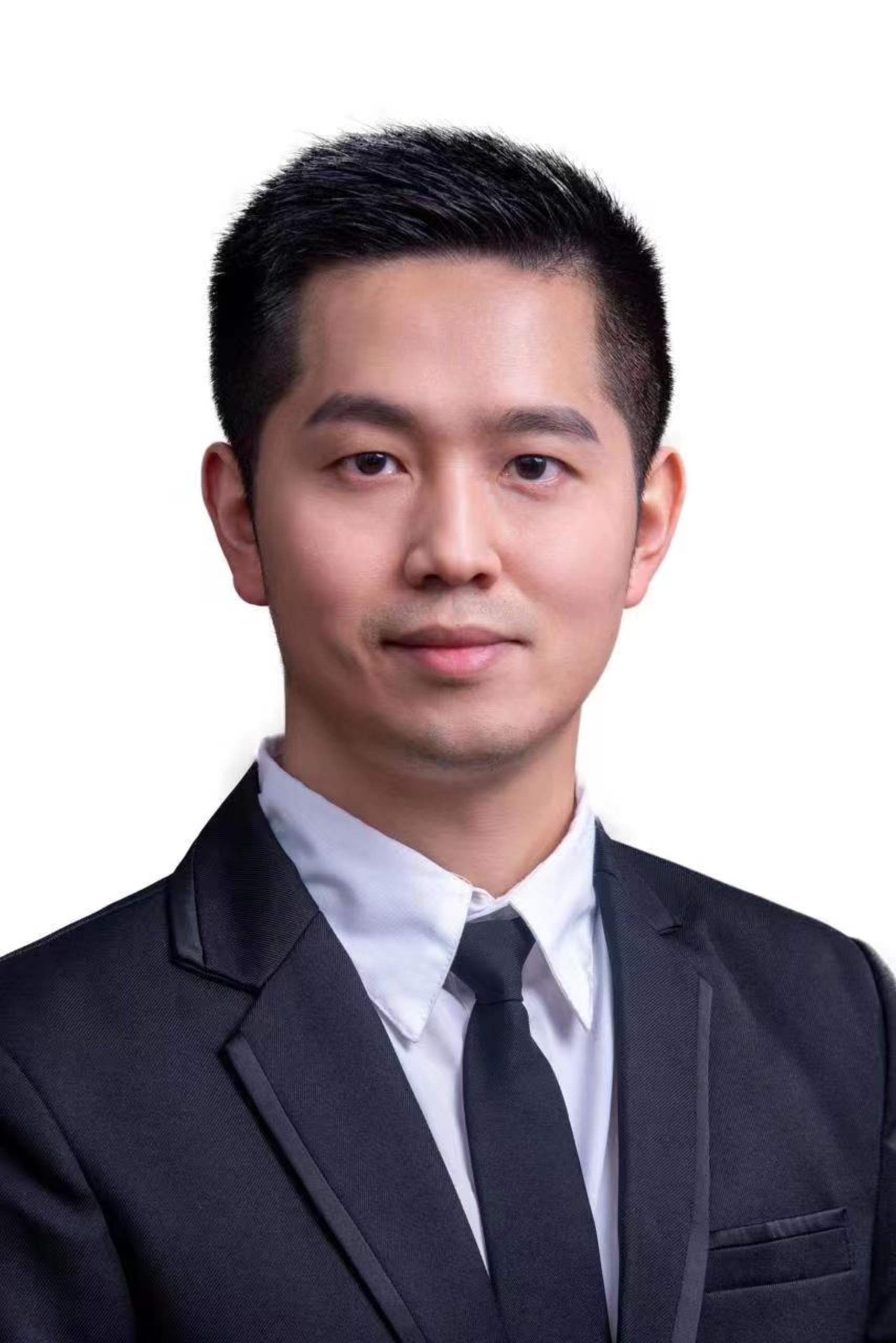}}]{Xiang Wang} received his B.Eng. and Ph.D. degrees in Electronic Science and Engineering from the National University of Defense Technology in 2007 and 2013, respectively. Currently, he is an Associate Professor at the College of Electronic Science and Engineering, National University of Defense Technology. 
His research interests include blind signal processing in radar and communication applications, as well as deep learning and machine learning for signal processing in these fields. Dr. Wang has published over 70 conference and journal papers, two of which have been recognized as ESI top 1\% highly cited papers. Additionally, he has received two best paper awards at international conferences and has served multiple times as a session chair for the IET Radar Conference.
\end{IEEEbiography}

\begin{IEEEbiography}[{\includegraphics[width=1in,height=1.25in,clip,keepaspectratio]{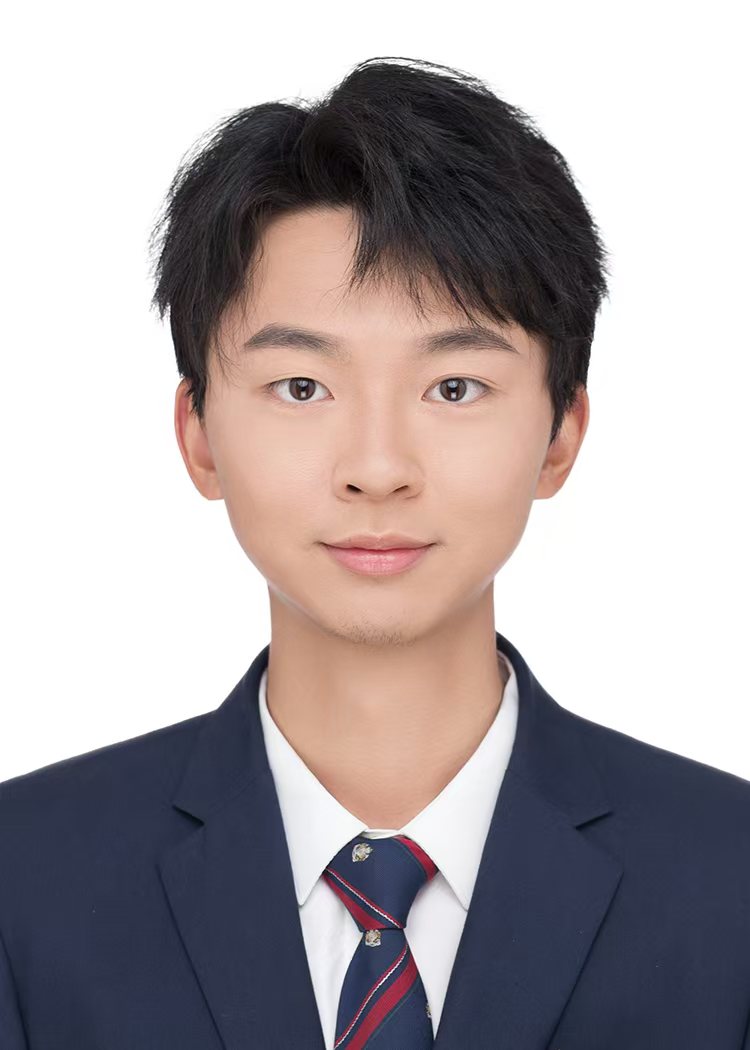}}]{Jingreng Lei} 
received the B.Eng. degree from Sun Yat-sen University, and the M.Phil. degree from
the Department of Electrical and Computer Engineering, The University of Hong Kong (HKU). He
is currently pursuing the Ph.D. degree with the Department of Electrical and Computer Engineering,
HKU. His research interests include sensing and communications, spatial intelligence.
\end{IEEEbiography}

\begin{IEEEbiography}[{\includegraphics[width=1in,height=1.25in,clip,keepaspectratio]{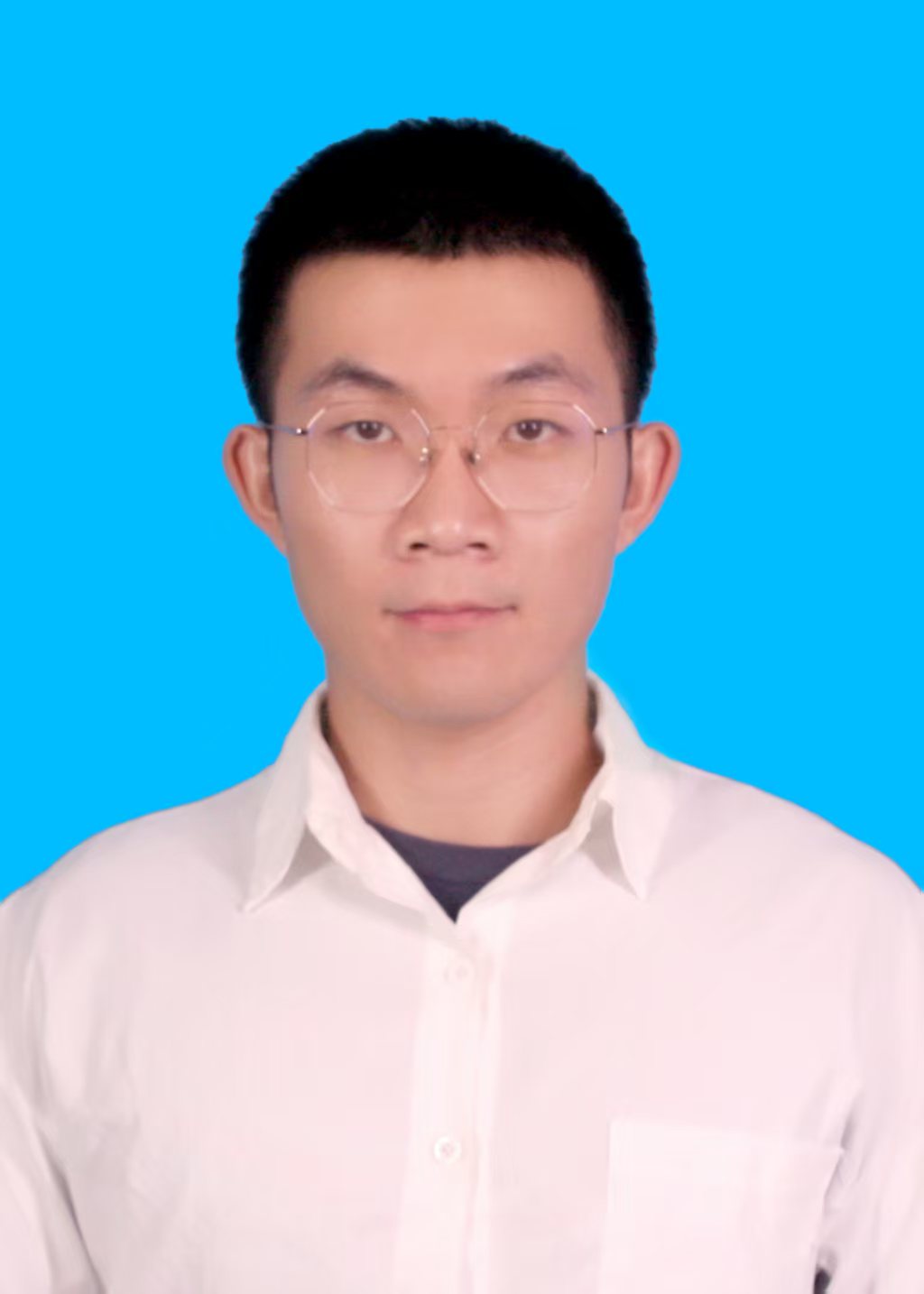}}]{Wanlong Zhang} 
received his B.Eng. and M.Phil. degrees in Electronic Science and Engineering from the National University of Defense Technology. 
He is currently pursuing a Ph.D. degree with the State Key Laboratory of Complex Electromagnetic Environment Effects on Electronics and Information Systems, National University of Defense Technology. 
His current research interests include wireless communications and deep learning. 
\end{IEEEbiography}

\begin{IEEEbiography}[{\includegraphics[width=1in,height=1.25in,clip,keepaspectratio]{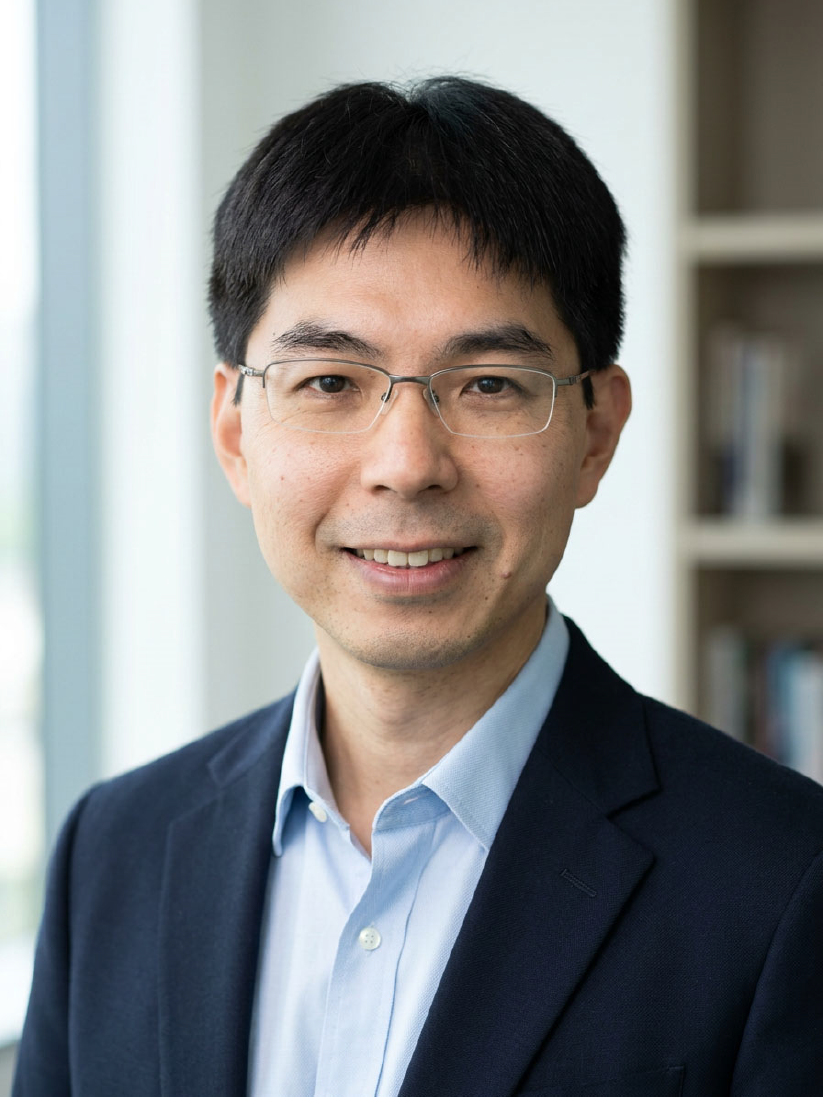}}]{Yik-Chung Wu} received the Ph.D. degree from Texas A\&M University, College Station, in 2005. After working as a Member of Technical Staff in Thomson Corporate Research, Princeton, NJ, for one year, he has been with The University of Hong Kong since 2006, currently as an Associate Professor. He was a visiting scholar at Princeton University, in summers of 2015 and 2017. His research interests are in general areas of signal processing, machine learning, and communication systems, and in particular Bayesian inference, distributed algorithms, and large-scale optimization.  Dr. Wu served as an Editor for IEEE Communications Letters, and IEEE Transactions on Communications. He is currently a Senior Area Editor for IEEE Transactions on Signal Processing, an Associate Editor for IEEE Wireless Communications Letters, and an Editor for Journal of Communications and Networks.  He received four best paper awards in international conferences, with the most recent one from IEEE International Conference on Communications (ICC) 2020.  He was a symposium chair for many international conferences, including IEEE International Conference on Communications (ICC) 2023 and IEEE Globecom 2025. He was elected the Best Editor of the year 2023 in IEEE Wireless Communications Letters. He is an elected member of IEEE signal processing society SPCOM Technical Committee (2025-2026), and an IEEE Distinguished Lecturer (Vehicular Technology Society 2025 class). 

\end{IEEEbiography}

\begin{IEEEbiography}[{\includegraphics[width=1in,height=1.25in,clip,keepaspectratio]{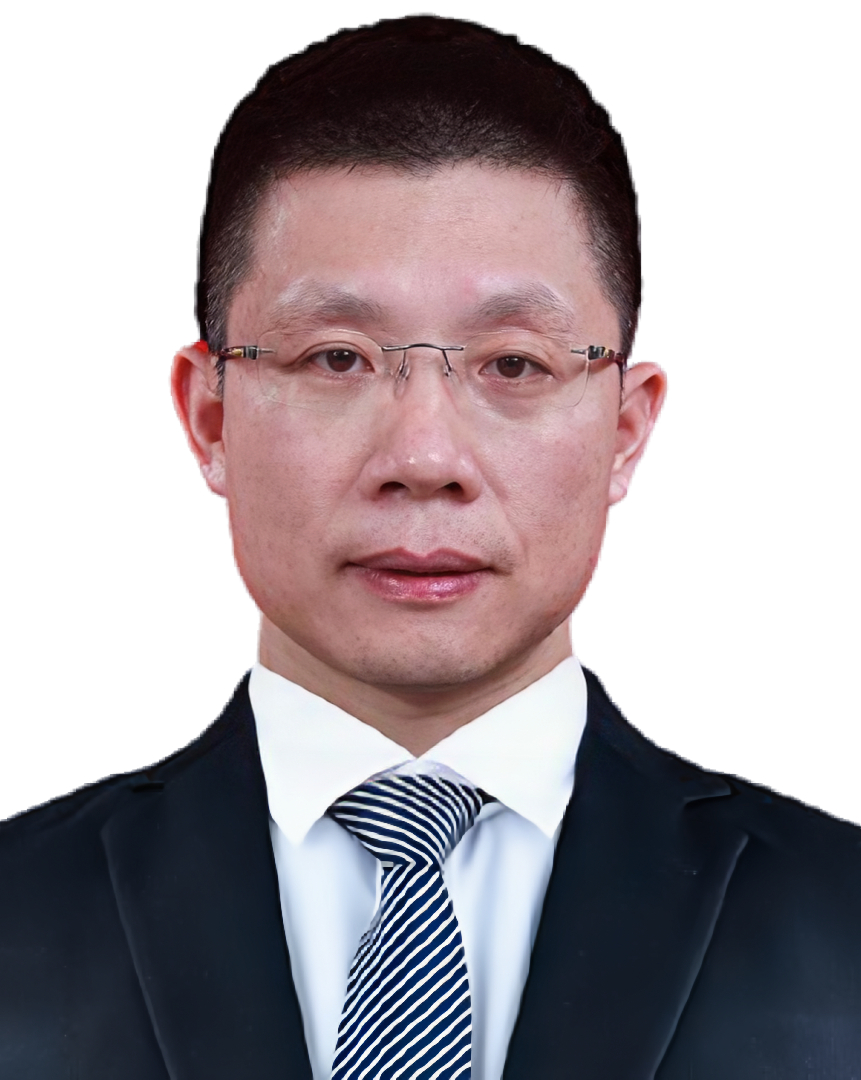}}]{Zhitao Huang}
received the B.Eng. and Ph.D. degrees in Information and Communication Engineering from the National University of Defense Technology. 
He is currently a professor at the College of Electronic Science and Engineering, National University of Defense Technology. 
His research interests include radar and communication signal processing and pattern recognition.
\end{IEEEbiography}

\clearpage
\newpage

\section*{APPENDIX}

\subsection{Proof of Theorem \ref{conclusion_approx} \label{proof_approx}}

The ideal format of the function ${\cal F^{\star}}$ can be decomposed into real-valued part and imaginary-valued part, expressed as 
\begin{equation}
{\cal F}^{\star} = {\cal R}({\cal F}^{\star}) + j{\cal I}({\cal F}^{\star}).
\end{equation}  
For RVNN $\hat{{\cal F}}^{\star}_{real}$, we have 
\begin{equation}
    \inf ||{\cal F}^{\star}-{ \hat{{\cal F}}^{\star}_{real}}|| \le E(C^{(a)}({\cal R}({\cal F}^{\star})),k)+||{\cal I}({\cal F}^{\star})||.
\end{equation}
For the CVNN $ \hat{{\cal F}^{\star}_{comp}}$, we have 
\begin{equation}
    \inf ||{\cal F}^{\star}-{ \hat{{\cal F}}^{\star}_{comp}}|| \le E(C^{(a)}({\cal F}^{\star}),2k).
\end{equation}
For the DualNN $ \hat{{\cal F}^{\star}_{dual}}$, we have 
\begin{equation}
\begin{array}{l}
    \inf ||{\cal F}^{\star}-{ \hat{{\cal F}}^{\star}_{dual}}|| 
    \\ 
    =\inf \left(||{\cal R}({\cal F}^{\star})-{ \hat{{\cal F}}^{\star}_{dual}}||
    + ||{\cal I}({\cal F}^{\star})-{\hat{{\cal F}^{\star}_{dual}}}|| \right)
    \\
    \le E(C^{(a)}({\cal R}({\cal F}^{\star})),k)+E(C^{(a)}({\cal I}({\cal F}^{\star})),k)
\end{array}.
\end{equation}

As the complexity of neural network grows unbounded as $k \xrightarrow{}\infty$, the approximation error bound for both DualNNs and CVNNs asymptotically converges to zero, enabling these architectures to achieve arbitrarily precise representations of the optimal target function across both real and imaginary domains. 
In contrast, RVNNs exhibit an inherent limitation since they only approximate the real component of the optimal function, irrespective of network complexity.

\subsection{Proof of Theorem \ref{conclusion_est} \label{proof_est}}

We provide a concise outline of the mathematical proof. The estimation error bound is governed by the spectral complexity $R_{\mathcal{A}}$, whose dominant term is $\prod_i \|\mathbf{A}_i\|_{\sigma}$. By deriving the closed-form optimal weight matrices for CVNN and DualNN and comparing their maximum eigenvalues, we obtain $\lambda_{\text{dual}} < \lambda_{\text{comp}}$ under mild conditions on the I/Q components. This inequality propagates through the spectral norm product, yielding $R_{\mathcal{A}}^{\text{dual}} < R_{\mathcal{A}}^{\text{comp}}$, and consequently the estimation error bound of DualNN is strictly lower than that of CVNN.
\color{black}

\begin{definition}[Spectral Complexity]
The spectral complexity $R_{\mathcal{A}}$ of a network $\cal F$ with weights ${\mathbf{A}}$ \cite{Bartlett2017} is defined as
\begin{equation}
R_{\mathcal{A}} := \left(\prod_{i=1}^{L} \rho_{i} \|\mathbf{A}_{i}\|_{\sigma}\right) \left(\sum_{i=1}^{L} \frac{\|\mathbf{A}_{i}^{\top} - \mathbf{M}_{i}^{\top}\|_{2,1}^{2/3}}{\|\mathbf{A}_{i}\|_{\sigma}^{2/3}}\right)^{3/2},
\label{eq_R_A}
\end{equation}
where $\|\cdot\|_{\sigma}$ denotes the spectral norm and $\|\cdot\|_{p, q}$ denotes the $(p, q)$ matrix norm such as $\|\mathbf{A}\|_{p, q} = \left\|\left(\left\|\mathbf{A}_{:, 1}\right\|_{p}, \ldots, \left\|\mathbf{A}_{:, m}\right\|_{p}\right)\right\|_{q}$. $\mathbf{M}_{i} \in \mathbb{R}^{d_{i} \times d_{i-1}}$ are reference matrices that can be chosen to adapt the bound to specific architectures (e.g., $\mathbf{M}_{i} = \mathbf{I}$ for ResNet).
\end{definition}

Eq.~\eqref{eq_R_A} consists of two parts, i.e., the Lipschitz constant of this neural network \( \prod_{i=1}^{L}\rho_{i}\left\|\mathbf{A}_{i}\right\|_{\sigma} \) and another factor related to the sum of quotients of weight matrix norms 
\( \left(\sum_{i=1}^{L}\frac{\left\|\mathbf{A}_{i}^{\top}-\mathbf{M}_{i}^{\top}\right\|_{2,1}^{2/3}}{\left\|\mathbf{A}_{i}\right\|_{\sigma}^{2/3}}\right)^{3/2} \). 
Since the two norms \({\left\|\mathbf{A}_{i}^{\top}-\mathbf{M}_{i}^{\top}\right\|_{2,1}^{2/3}}\) and \( {\left\|\mathbf{A}_{i}\right\|_{\sigma}^{2/3}}\) are equivalent, the factor remains in the interval \([C_{1}, C_{2}]\) for some constants \( C_{1} \) and \( C_{2} \), whose change is minor \cite{Bartlett2017}.
Also, for the Lipschitz constant of the neural network, the Lipschitz constants of activation functions \( \rho_{i} \) remain unchanged.
Hence, the part \( \left(\prod_{i=1}^{L}\left\|\mathbf{A}_{i}\right\|_{\sigma}\right) \) 
dominates the change in \( R_{\mathcal{A}} \), where
$\left\|\mathbf{A}_i\right\|_{\sigma}$ is the maximum singular value, i.e., the square of the maximum eigenvalue of the weight matrix $\mathbf{A}_i$, expressed as 
\(\left\|\mathbf{A}_i\right\|_{\sigma}=\sigma_{max}=\sqrt{\lambda_{max} (\mathbf{A}_i^{\text{H}}\mathbf{A}_i)}\).

Let $\mathbf{z}_{in} = \mathbf{z}_{in}^{(I)} + j\mathbf{z}_{in}^{(Q)}$ and $\mathbf{z}_{out}=\mathbf{z}_{out}^{(I)} + j\mathbf{z}_{out}^{(Q)}$ be the input and output of the $i$th layer. Define $\alpha_{out}^{(I)} = \| \mathbf{z}_{out}^{(I)} \|^2$, $\alpha_{in}^{(I)} = \| \mathbf{z}_{in}^{(I)} \|^2$, $\alpha_{in}^{(Q)} = \| \mathbf{z}_{in}^{(Q)} \|^2$, $\alpha_{out}^{(Q)} = \| \mathbf{z}_{out}^{(Q)} \|^2$, $ \beta=({\mathbf{z}_{in}^{(I)}})^{\top}\mathbf{z}_{in}^{(Q)}=(\mathbf{z}_{in}^{(Q)})^{\top}\mathbf{z}_{in}^{(I)}$, and $\gamma=({\mathbf{z}_{out}^{(I)}})^{\top}\mathbf{z}_{out}^{(Q)}=(\mathbf{z}_{out}^{(Q)})^{\top}\mathbf{z}_{out}^{(I)}$.
Next, we analyze the $\mathbf{A}_i$ of CVNN and DualNN to compare their spectral complexity.

\begin{enumerate}
    \item \textbf{Case 1 (CVNN)}:
        
        For the CVNN, we have 
            \begin{equation}
               \mathop {\min }\limits_{\mathbf{A}_i}  J({\mathbf{A}_i}) = \left\| {\mathbf{A}_i}{\mathbf{z}_{in}}-\mathbf{z}_{out} \right\|^2.
            \end{equation}
            Setting $\bigl( \partial J(\mathbf{A}_i) \bigr) / \bigl( \partial \mathbf{A}_{i}^{*} \bigr)=\bf{0}$, we can obtain the solution for $\mathbf{A}_i$ as follows.
            \begin{equation}
                {{\mathbf{A}_i}} = {{\mathbf{z}_{out}}}{{(\mathbf{z}_{in})}^{\rm{H}}}{\left( {{\mathbf{z}_{in}}^{\rm{H}}}{{\mathbf{z}_{in}}} \right)^{-1}}.
                \label{eq_ai_complex}
        \end{equation}
         Based on Eq.~\eqref{eq_ai_complex}, $\mathbf{A}_i^{\rm{H}} \mathbf{A}_i$ can be written as:
         \begin{equation}
            \mathbf{A}_i^{\rm{H}} \mathbf{A}_i = \frac{\|\mathbf{z}_{\mathrm{out}}\|^2}{\|\mathbf{z}_{\mathrm{in}}\|^4} \mathbf{z}_{\mathrm{in}} \mathbf{z}_{\mathrm{in}}^{\rm{H}}.
            \label{eq:AHA}
        \end{equation}
        By calculating the maximum eigenvalue of \eqref{eq:AHA}, we obtain
        \begin{equation}
            \lambda_{comp} =  \frac{\alpha_{out}^{(I)}+\alpha_{out}^{(Q)}}{\alpha_{in}^{(I)}+\alpha_{in}^{(Q)}}.
            \label{lamda_complex_source}
        \end{equation}
        
    \item \textbf{Case 2 (DualNN)}:

            For the DualNN, we have 
                \begin{equation}
                   \mathop {\min }\limits_{\mathbf{A}_i}  J({\mathbf{A}_i}) = \left\| {\mathbf{A}_i}{\mathbf{z}_{in}^{(I)}}-\mathbf{z}_{out}^{(I)} \right\|^2 + \left\| {\mathbf{A}_i}{\mathbf{z}_{in}^{(Q)}}-\mathbf{z}_{out}^{(Q)} \right\|^2.
                \end{equation}
            Setting $\bigl( \partial J(\mathbf{A}_i) \bigr) / \bigl( \partial \mathbf{A}_i \bigr)=\bf{0}$, we can obtain the solution for $\mathbf{A}_i$ as follows:
            \begin{equation}
            \begin{array}{cc}
                \mathbf{A}_i = \left( {{\mathbf{z}_{out}^{(I)}}}{{(\mathbf{z}_{in}^{(I)})}^{\top}}
                +{{\mathbf{z}_{out}^{(Q)}}}{{(\mathbf{z}_{in}^{(Q)})}^{\top}}\right) \times \\{\left( {\left({\mathbf{z}_{in}^{(I)}}\right)^{\top}}{{\mathbf{z}_{in}^{(I)}}} + 
                {\left({\mathbf{z}_{in}^{(Q)}}\right)^{\top}}{{\mathbf{z}_{in}^{(Q)}}}
                \right)^{-1}}.
            \end{array}
                \label{eq:A_i_final}
            \end{equation}

                Based on Eq.~\eqref{eq:A_i_final}, $\mathbf{A}_i^\top \mathbf{A}_i$ can be written as:
                \begin{equation}\label{eq:AHA_dual}
                    \mathbf{A}_i^\top\mathbf{A}_i = \frac{\alpha_{out}^{(I)}\mathbf{z}_{in}^{(I)}(\mathbf{z}_{in}^{(I)})^{\top}+\gamma\mathbf{C}+\alpha_{out}^{(Q)}\mathbf{z}_{in}^{(Q)}(\mathbf{z}_{in}^{(Q)})^{\top}}{(\alpha_{in}^{(I)}+\alpha_{in}^{(Q)})^2},
                \end{equation}
                where $\mathbf{C} = \mathbf{z}_{in}^{(I)}(\mathbf{z}_{in}^{(Q)})^{\top}+\mathbf{z}_{in}^{(Q)}(\mathbf{z}_{in}^{(I)})^{\top}$.
                Following the above line, we obtain the maximum eigenvalue: 
                \begin{equation}
                    \lambda_{dual} = \frac{L + \sqrt{D}}{2(\alpha_{in}^{(I)} + \alpha_{in}^{(Q)})^2},
                \end{equation}
                where $L=\alpha_{in}^{(I)}\alpha_{out}^{(I)}+\alpha_{in}^{(Q)}\alpha_{out}^{(Q)}+2\beta\gamma$ and $D=\sqrt{L^2-4(\alpha_{in}^{(I)}\alpha_{in}^{(Q)}-\gamma^2)(\alpha_{out}^{(I)}\alpha_{out}^{(Q)}-\beta^2)}$
\end{enumerate}

We now examine the relationship between $\lambda_{comp}$, and $ \lambda_{dual}$.

Since  $\alpha_{in}^{(I)}\alpha_{in}^{(Q)}>\gamma^2$ and $\alpha_{out}^{(I)}\alpha_{out}^{(Q)}>\beta^2$, it follows that
\begin{equation}
    L+\sqrt{D}<L+\sqrt{L^2}=2L .
\end{equation}
Consequently, $\lambda_{dual}$ satisfies
\begin{equation}
    \lambda_{dual} < \frac{L}{(\alpha_{in}^{(I)}+\alpha_{in}^{(Q)})^2}.
    \label{lamda_dual}
\end{equation}
To analyze $\lambda_{comp}$, we multiply the numerator and denominator of Eq.~\eqref{lamda_complex_source} by $\alpha_{in}^{(I)}+\alpha_{in}^{(Q)}$, yielding
\begin{align}
\lambda_{comp}&=\frac{\alpha_{out}^{(I)}\alpha_{in}^{(I)}+\alpha_{out}^{(Q)}\alpha_{in}^{(Q)}+\alpha_{out}^{(I)}\alpha_{in}^{(Q)}+\alpha_{out}^{(Q)}\alpha_{in}^{(I)}}{(\alpha_{in}^{(I)}+\alpha_{in}^{(Q)})^2}\\
&=\frac{L+\alpha_{out}^{(I)}\alpha_{in}^{(Q)}+\alpha_{out}^{(Q)}\alpha_{in}^{(I)}-2\beta\gamma}{(\alpha_{in}^{(I)}+\alpha_{in}^{(Q)})^2}.
\label{lamda_complex}
\end{align}
Furthermore, when the inequalities $\alpha_{in}^{(I)} > \beta$, $ \alpha_{in}^{(Q)} > \beta$, $\alpha_{out}^{(I)} > \gamma$, and
$\alpha_{out}^{(Q)} > \gamma$, we can derive that the numerator of $\lambda_{comp}$ strictly exceeds $L$. Combining this observation with \eqref{lamda_dual} and \eqref{lamda_complex}, we conclude
\begin{equation}
    \lambda_{comp}>\frac{L}{(\alpha_{in}^{(I)}+\alpha_{in}^{(Q)})^2}>\lambda_{dual}.
\end{equation}

Building upon the results from Eq.~\eqref{eq_R_A}, we obtain the following key inequality
\begin{equation}
\begin{array}{cc}
    R_{\cal A}^{dual} < R_{\cal A}^{comp}.
\end{array}
\end{equation}
When examining the estimation errors through Lemma 3, and Corollary 1, we establish

\begin{equation}
\begin{array}{cc}
\sup \epsilon_{\mathrm{dual}}^{\mathrm{est}} < \sup \epsilon_{\mathrm{comp}}^{\mathrm{est}}.
\end{array}
\end{equation}

\end{document}